\documentclass{article}
\usepackage{hyperref}       
\usepackage{url}      
\usepackage{booktabs}       
\usepackage{amsfonts}       
\usepackage{nicefrac}    
\usepackage[square,numbers]{natbib}
\usepackage{microtype}      
\usepackage{xcolor}      
\usepackage{fullpage}
\usepackage{hyperref}       
\usepackage{url}            
\usepackage{booktabs}       
\usepackage{amsfonts}       
\usepackage{nicefrac}       
\usepackage{microtype}      
\usepackage{xcolor}         
\usepackage{paralist,amsmath, amssymb}
\usepackage{algorithm,algorithmic}
 \usepackage{graphicx}
 \usepackage{float}
 \usepackage{pdfpages}
 \usepackage{tikz}
\usepackage{subfigure} 
\usepackage{xcolor}
\usepackage{graphics}
\usepackage{units}
 \usepackage{amssymb, multirow, paralist, color}
 \usepackage{amsthm}
 \usepackage{textcase, booktabs,makecell}
 \usepackage{thmtools,thm-restate}
 \usepackage{mathtools}
\usepackage{enumitem}
\usepackage{titlesec}
\usepackage{tikz}
\usetikzlibrary{positioning}

 \usepackage{thmtools,thm-restate}
 \usepackage{mathtools}
\usepackage{enumitem}
\def \y {\mathbf{y}}

\def \x {\mathbf{x}}
\def \g {\mathbf{g}}

\def \z {\mathbf{z}}
\def \u {\mathbf{u}}
\def \r {\mathbf{r}}

\def \w {\mathbf{w}}
\def \s {\mathbf{s}}
\def \R {\mathbb{R}}
\def \m {\mathbf{m}}

\def \A {\mathcal{A}}

\def \v {\mathbf{v}}

\def \p {\mathbf{p}}

\def \a {\mathbf{a}}

\def \B {\mathbf{B}}

\def \A {\mathbf{A}}

\def \V {\mathcal{V}}

\def \m {\mathbf{m}}

\def \K {\mathcal{K}}

\def \V {\mathcal{V}}

\def \B {\mathcal{B}}

\def \S {\mathcal{S}}
\def \wn {\w^*_{\|\cdot\|}}

\newcounter{protocoll}
\makeatletter
\newenvironment{protocoll}[1][htb]{%
  \let\c@algorithm\c@protocoll
  \renewcommand{\ALG@name}{Protocol}
  \begin{algorithm}[#1]%
  }{\end{algorithm}
}
\makeatother

\newtheorem{ass}{Assumption}
\newtheorem{lemma}{Lemma}
\newtheorem{theorem}{Theorem}

\newtheorem{definition}{Definition}

\DeclareMathOperator*{\argmin}{argmin}
\DeclareMathOperator*{\argmax}{argmax}
\usepackage{microtype}
\usepackage{subfigure}
\usepackage{booktabs} 
\def \dl {\boldsymbol{\delta}}

\begin{document}

\begin{center}

{\bf{\Large{Faster Margin Maximization Rates for  Generic \\ and Adversarially Robust Optimization Methods}}}

\vspace*{.2in}

{\large{
\begin{tabular}{cccc}
 Guanghui Wang$^1$, Zihao Hu$^1$, Claudio Gentile$^2$\\
 Vidya Muthukumar$^{3,4}$, Jacob Abernethy$^{1,5}$  \\
\end{tabular}}}

\vspace*{.05in}

\begin{tabular}{c}
$^1$College of Computing, Georgia Institute of Technology\\
$^2$Google Research, New York\\
$^3$School of Electrical and Computer Engineering, Georgia Institute of Technology\\
$^4$School of Industrial and Systems Engineering, Georgia Institute of Technology\\
$^5$Google Research, Atalanta\\

  \texttt{\{gwang369,zihaohu,vmuthukumar8\}@gatech.edu}, \texttt{\{abernethyj,cgentile\}@google.com} 
\end{tabular}

\vspace*{.2in}
\date{}

\end{center}

\abstract{First-order optimization methods tend to inherently favor certain solutions over others when minimizing an underdetermined training objective that has multiple global optima. This phenomenon, known as \emph{implicit bias}, plays a critical role in understanding the generalization capabilities of optimization algorithms. Recent research has revealed that in separable binary classification tasks gradient-descent-based methods exhibit an implicit bias for the $\ell_2$-maximal margin classifier. Similarly, generic optimization methods, such as mirror descent and steepest descent, have been shown to converge to maximal margin classifiers defined by alternative geometries. While gradient-descent-based algorithms provably achieve {\em fast} implicit bias rates, corresponding rates in the literature for generic optimization methods are relatively slow. To address this limitation, we present a series of state-of-the-art implicit bias rates for mirror descent and steepest descent algorithms. Our primary technique involves transforming a generic optimization algorithm into an online optimization dynamic that solves a regularized bilinear game, providing a unified framework for analyzing the implicit bias of various optimization methods. Our accelerated rates are derived by leveraging the regret bounds of online learning algorithms within this game framework. We then show the flexibility of this framework by analyzing the implicit bias in  {\em adversarial training}, and again obtain significantly improved convergence rates. 
}

\section{Introduction}\label{sec1}

The optimization objective involved in the training 
of modern (overparameterized) Machine Learning (ML) models is typically \emph{underdetermined}, meaning that it presents infinitely many global optima.
Yet, even unregularized first-order optimization methods are observed to converge to solutions that generalize well to test data, as multiple empirical studies have repeatedly confirmed (e.g., \citet{zhang2021understanding,neyshabur2014search}).
Moreover, robust optimization methods such as \emph{adversarial training} are empirically observed to achieve solutions that generalizes well even in the presence of adversarial perturbations of data (e.g., \citet{madry2018towards,li2020implicit}).
These observations have spurred interest in what is commonly called the \emph{implicit bias} of these optimization methods: namely, \emph{(a) which solution (i.e. global minimum) is favored by a particular first-order optimization method, and (b) at what speed do the parameters of the model converge to this solution?}

This paper addresses these questions in the regime of underdetermined linear classification.
When the training data is separable, the optimization objective is typically the unregularized empirical risk measured through a convex loss function $r(\cdot)$ that acts as a suitable surrogate to the discontinuous $0-1$ loss (see Equation~\eqref{eqn:the ERM function} for a formal definition).
Specializing to exponentially-tailed loss functions (which include the popular logistic loss), we have a rich asymptotic theory that addresses question (a), linking the implicit bias of an optimization method to its geometry.
The pioneering works~\citet{ji2018risk,soudry2018implicit} first characterized the implicit bias of gradient descent (GD) by the solution that maximizes the (normalized) margin measured in Euclidean distance; this is commonly called the $\|\cdot\|_2$-maximal margin classifier. Subsequently,~\citet{gunasekar2018characterizing} showed that the implicit bias of the steepest descent algorithm with respect to a general norm $\|\cdot\|$ is the corresponding $\|\cdot\|$-maximal margin classifier, and~\citet{sun2022mirror} showed that the implicit bias of the mirror descent algorithm with the potential $\|\cdot\|_q^q$ (for $q > 1$) is the corresponding $\|\cdot\|_q$-maximal margin classifier. Therefore, these richer families of algorithms can adapt to different data geometries by varying the choice of norm or potential function.
On the side of 
Adversarial Training (AT), GD augmented with adversarial perturbations in a bounded $\ell_s$-norm (called $\ell_s$-AT as shorthand) is known to converge to the maximum $(2,s)$-mixed-norm margin classifier, which could yield improved robustness properties depending on the choice of $s$~\cite{charles2019convergence,li2020implicit}.
The choice of $s$, and therefore the choice of the optimization algorithm, affects the entire nature of the eventual solution, thus playing a pivotal role in robustness.

Question (b), i.e., the rate of \emph{parameter or margin convergence} to these implicit biases, has been addressed in part for specific optimization methods, but the picture remains incomplete and complex.
Even for the special case of Euclidean geometry, parameter or margin convergence analyses present several technical challenges due to the non-smoothness of the margin function, the presence of multiple global minima of the original optimization objective function, and the fact that, for exponentially tailed losses such as the logistic loss, all of these minima are attained at infinity, meaning that the \emph{direction} of the parameter is what needs to be considered~\cite{dudik2022convex}.
The initial works on GD only established a slow rate of $\mathcal{O}\left(\frac{\log n}{\log T}\right)$, where $T$ is the time horizon and $n$ is the cardinality of the dataset~\cite{soudry2018implicit,ji2018risk}.
Since then, GD was shown to attain better rates of $\mathcal{O}\left(\frac{\log n + \log T}{\sqrt{T}}\right)$~\cite{nacson2019convergence} and $\mathcal{O}\left(\frac{\log n}{T}\right)$~\cite{ji2021characterizing} with a more aggressive step size schedule, and the fastest known rate of $\mathcal{O}\left(\frac{\log n}{T^2}\right)$ additionally imbues GD with momentum or Nesterov's acceleration~\cite{ji2021fast,wang2022accelerated}.
The picture of margin convergence rates remains much more limited for generic optimization methods, with the fastest known rate being $\mathcal{O}\left(\frac{\log n + \log T}{\sqrt{T}}\right)$ for steepest descent methods~\cite{gunasekar2018characterizing} and $\mathcal{O}\left(\frac{\log n}{T^{1/4}}\right)$ for mirror descent methods\footnote{This result comes with a caveat that the potential function needs to be strongly convex \emph{and} strongly smooth with respect to a general norm, thus limiting it to the Euclidean geometry.}~\cite{li2021implicit}.
For GD augmented with adversarial training, the fastest known rate is $\mathcal{O}\left(\frac{\text{poly}(n)}{\sqrt{T}}\right)$ for $\ell_2$-perturbations of the data, with all other perturbation norms (i.e.,~$s \neq 2$) yielding a much slower $\mathcal{O}\left(\frac{\log n}{\log T}\right)$ rate~\cite{li2020implicit}.
The required analyses for these methods are generally quite complex and idiosyncratic, and each tends to rely on specific details of the particular optimization procedure.

\subsection{Main results and techniques}
In this paper, we provide the fastest known rates for margin maximization and parameter convergence for generic optimization methods and adversarially robust optimization methods run with the exponential loss function, as summarized below.
\begin{itemize}
    \item First, we study a \emph{weighted-average} version of mirror descent with the squared $\ell_q$-norm $\frac{1}{2} \|\cdot\|_q^2$ as the potential for $q \in (1,2]$.
    We show that with an appropriately chosen step size, the algorithm achieves a faster $\|\cdot\|_q$-margin maximization rate on the order of  $\textstyle \mathcal{O}\left(\frac{\log n\log T}{(q-1)T}\right)$. We also further improve the rate to $\textstyle \mathcal{O}\left(\frac{1}{T(q-1)}+\frac{\log n \log T}{T^2}\right)$ with a more aggressive step size. When $q=2$, the algorithm reduces to average  GD, and our rate $\mathcal{O}\left(\frac{1}{T}+\frac{\log n \log T}{T^2}\right)$ is a $\log n$-factor tighter than the $\textstyle \mathcal{O}\left(\frac{\log n}{T}\right)$ rate of the last-iterate of GD \citep{ji2021characterizing}. 
   \item Next, for the steepest descent  with strongly convex norm, we show the margin maximization rate can  be improved from $\mathcal{O}\left(\frac{\log n+\log T}{\sqrt{T}}\right)$ to $\mathcal{O}\left(\frac{\log n}{T}\right)$.
   \item We then demonstrate that an even faster $\mathcal{O}\left(\frac{\log n}{T^2(q-1)}\right)$ $\|\cdot\|_q$-margin maximization rate can be achieved in two ways: (a) mirror descent with Nesterov acceleration, or (b) steepest descent with extra gradient and momentum.
   \item Moving to adversarial training, we show that for $s\in(1,2]$, $\ell_s$-AT utilizing Normalized Gradient Descent (NGD) converges at a rate of $\mathcal{O}\left(\frac{\log n}{{T}}\right)$ towards the $(2,s)$-mix-norm max-margin classifier. 
   \item When further equipped with Nesterov-style acceleration, $\ell_s$-AT achieves a faster $\mathcal{O}\left(\frac{\log n}{T^2}\right)$ rate 
   { for $s\in(1,2]$, and $\mathcal{O}\left(\frac{\log n}{T}\right)$ for $s>2$.} Somewhat surprisingly, the fastest rates of AT end up matching those of optimization on clean data, at least in the case of linear classification.
\end{itemize}

We summarized our main results in Table \ref{tab:my-table}.
The essential premise for our approach is that \emph{Empirical Risk Minimization (ERM) with generic optimization methods can be equivalently viewed as solving a {regularized} bilinear game with online learning dynamics}. Within this framework, we design new pairs of online learning methods whose outputs (and, by extension, the outputs of the corresponding generic optimization methods) automatically maximize the margin. The convergence rates are determined by the time-averaged regret bounds of these online learning algorithms \emph{when played against each other}, which turn out to be much faster than the worst-case $\mathcal{O}(1/\sqrt{T})$ rate.
In addition to yielding these faster rates, the convergence analysis is often very simple---indeed, the main nontriviality in our approach is the identification of the correct pair of online learning dynamics, and proving their equivalence.
A block diagram illustration of this game framework is provided in Fig.~\ref{fig:framework}.

\begin{figure}
\centering
\begin{minipage}{0.9\textwidth}
\begin{table}[H]
\centering
\resizebox{0.99\textwidth}{!}{
\begin{tabular}{@{}ccccc@{}}
\toprule
\textbf{Algo.} &
  \textbf{Ref.} &
  \textbf{Rate} &
  \textbf{Step size/notes} \\ \midrule
\multirow{4}{*}{Mirror Descent} &
  Theorem \ref{thm:mirror margin} &
 $\mathcal{O}\left(\frac{\log T}{T(q-1)}\right)$  &
  $\frac{1}{\text{function value}}$ \\  
  \cmidrule(l){2-5} 
 &
  Theorem \ref{cor:u2222} &
  $\mathcal{O}\left(\frac{1}{T} + \frac{\log T}{T^2(q-1)}\right)$ &
  $\frac{t}{\text{function value}}$ \\ \cmidrule(l){2-5} 
 &
  Theorem \ref{thm:mirror momentum} &
  $\mathcal{O}\left(\frac{\V_T}{T^2(q-1)}+\frac{\log n \log T}{T^2}\right)$ &
  Momentum \\  
  \midrule
\multirow{1}{*}{Steepest Descent} &
    Theorem \ref{thm:steep:acc::main} &
  $\mathcal{O}\left(\frac{\log n}{{T}}\right)$ &
   $\frac{1}{\text{function value}}$\\ 
   \midrule
\multirow{2}{*}{Accelerated Algorithms} &
  \multirow{2}{*}{Theorem \ref{thm:accel}} &
  \multirow{2}{*}{$\mathcal{O}\left(\frac{\log n}{T^2(q-1)}\right)$} &
  \begin{tabular}[c]{@{}c@{}}MD with \\ Nesterov acceleration\end{tabular} \\ \cmidrule(l){5-5} 
 &
   &
   &
  \begin{tabular}[c]{@{}c@{}}SD with extra \\ gradient and momentum\end{tabular} \\ \midrule
  \multirow{5}{*}{$\ell_s$-AT} 
 &
  Theorem \ref{thm:GDBAT} &
  $\mathcal{O}\left(\frac{\log n}{T}\right)$, for $s\in(1,2]$ &
  $\ell_s$-AT with GD \\ \cmidrule(l){2-5} 
 &
  \multirow{2}{*}{Theorem \ref{thm:acccccc}}&
  $\mathcal{O}\left(\frac{\log n}{T^2}\right)$, for $s\in(1,2]$ &
   \multirow{2}{*}{$\ell_s$-AT with Accelerated Methods }\\
   \cmidrule(l){5-5} 
 &
    &
$\mathcal{O}\left(\frac{\log n}{{T}}\right)$, for $s\in(2,\infty)$   &\\
  
  \bottomrule
\end{tabular}}
\caption{Fast margin maximization rates for generic optimization methods and adversarial training.}
\label{tab:my-table}
\end{table}
\end{minipage}
\end{figure}

\citet{wang2022accelerated} were the first to draw parallels between \emph{Nesterov-accelerated GD} for ERM and solving the bilinear game through an online dynamic. However, it was still open whether this kind of analysis suited other optimization geometries. 
We reveal, through a simpler, streamlined and unified analysis, that the game framework can in fact encompass implicit bias analyses for a range of generic optimization methods. 
We also derive auxiliary results beyond the main results mentioned above, summarized below.
\begin{itemize}   
    \item By selecting suitable online learning algorithms, we obtain a momentum-based data-dependent MD algorithm with an $\mathcal{O}\left(\frac{\V_T}{T^2(q-1)}+\frac{\log n \log T}{T^2}\right)$ $\|\cdot\|_q$-margin maximization rate, where $\V_T=\sum_{t=2}^{T}\|\p_t-\p_{t-1}\|^2_1$ is the path-length of a series of distributions on the training data $\p_t$. In the worst case, this reduces to the margin maximization rate of MD, but this could be much tighter if $\V_T$ is sublinear in $T$.
    \item Apart from margin maximization rates, we also bound the corresponding directional error, i.e., the $\ell_q$-distance between the maximal margin classifier and the normalized output of the generic methods, which are also controlled by the regret bounds of two-players playing against each other. This kind of convergence rates are new for most of the generic methods. In general, we show the directional errors are typically a square-root factor worse than the margin maximization rates. 
    \item For steepest descent, by setting the norm to the general norm $\|\cdot\|$ and the $\ell_2$-norm respectively, we can recover the algorithms and theoretical guarantees in \citet{nacson2019convergence,ji2021characterizing} under the game framework. This implies that these algorithms can also be viewed as solving a \emph{regularized} bilinear game using online learning algorithms, offering a deeper understanding of the role of implicit bias in optimization methods.
\end{itemize}
As we can see, this self-contained description of the two-player bilinear game framework effectively captures a gamut of generic optimization methods on clean data.
On the other hand, the more complex procedures of robust optimization and adversarial training, which involve an additional step of selecting perturbations on input data, does not fit a bilinear game framework.
To address this challenge, we extend our game framework to a novel general-sum, regularized multilinear game with \emph{multiple players} to accommodate the perturbation process.
In particular, we add $n$ new players to this game, each corresponding to individual perturbations on training examples.

Compared to the preceding two-player bilinear setting, the inclusion of the additional $n$ players
presents new challenges. For instance, when analyzing $\ell_p$-AT with NGD, the online algorithms that are needed for proving algorithm equivalence {suffer a divergent average regret if we naively apply standard bounds. We overcome this hurdle by providing a novel and much tighter regret bound for our specific problem (see the proof of Theorem \ref{thm:GDBAT} for details).} Identifying the correct low-regret online methods that offer algorithm equivalence for the other two methods is also non-trivial.

\subsection{Additional Related Work}
Our discussion on the implicit bias and its convergence rates 
has so far been restricted to classification-oriented losses $r(z)$, such as the logistic loss and exponential loss, that attain their minimum at infinity, and optimization geometries that are strongly convex.
The implicit bias of regression problems, where the square loss $\ell(z) = z^2$ is used, has also been studied.
As indicated in~\citet{gunasekar2018characterizing,sun2022mirror,vardi2022implicit}, the analysis for square loss is ``fundamentally different", since the loss is not minimized at infinity.
Within the context of classification tasks and classification loss functions, the rates of implicit bias convergence have also been studied for AdaBoost (to the maximum-$\ell_1$-margin classifier at a $\mathcal{O}(\frac{1}{\sqrt{T}})$ rate~\cite{telgarsky2013margins}) and adaptive optimization methods such as Adam~\cite{gunasekar2018characterizing,wang2022does}.

The strategy of solving a zero-sum game using online learning algorithms playing against each other has been extensively studied, primarily through the lens of \emph{independent learning agents} (e.g., \cite{Predictable:NIPS:2013,daskalakis2018training,oftl_md,daskalakis2019last,ICML'22:TVgame}). In contrast, our central motivation and challenge lies in identifying the exact equivalent forms of generic optimization algorithms under the regularized bilinear (or multilinear) game dynamic. Our framework is also motivated by the line of research that employs the \emph{Fenchel-game} to elucidate commonly used convex optimization methods \citep{abernethy2018faster,oftl_md,wang2021no}. However, our framework diverges significantly from these approaches. These works focus on the convergence of the optimization problem itself, while our framework emphasizes that the choice of optimization algorithm, which solely targets the minimization of empirical risk, has a significant impact on maximizing the margin, which we might view as an ``algorithmic externality.'' It is important to emphasize that margin guarantees can not arise from convergence of the ERM objective alone, as there are typically multiple global minima in ERM minimization. Our analysis also considers an entirely different min-max problem than that of the Fenchel game \citep{wang2021no}; thus, the correspondences we establish between optimization algorithms and online dynamics also differ. Finally, we note that previous work has also analyzed the implicit bias through direct primal optimization analyses (e.g., \cite{nacson2019convergence,sun2022mirror}) or using a dual perspective (e.g., \cite{ji2021characterizing,ji2021fast}). For the former analyses, it is unclear whether and how faster rates can be obtained. For the latter, it remains an open question how to extend the framework beyond the $\ell_2$-geometry, which in some sense was the motivation for the present work.

Finally, the effectiveness of adversarial training in enhancing model robustness against adversarial attacks has been widely studied in practice~\cite{zhang2019theoretically,carmon2019unlabeled,raghunathan2020understanding,rice2020overfitting,sanyal2020benign}. However, it often comes with increased computational costs~\cite{madry2018towards}, prompting researchers to explore the convergence rates even in simpler linear settings.
The previously obtained slow rates~\cite{charles2019convergence,li2020implicit} left open the possibility that AT was indeed slower than optimization on clean data; our results show that this is in fact not the case.

\section{Preliminaries}
We first describe our basic setting, along with standard assumptions and definitions.\\[-3mm]

\noindent \textbf{Notation}\  \ We use lower case bold face letters $\x,\y$ to denote vectors, lower case letters $a,b$ to denote scalars, and upper case bold face letters $\A,\mathbf{B}$ to denote matrices. For a vector $\x\in\R^d$, we use $x_i$ to denote the $i$-th component of $\x$. For a matrix $\A\in\R^{n\times d}$, let $\A_{(i,:)}$ be its $i$-th row, $\A_{(:,j)}$ the $j$-th column, and $\A_{(i,j)}$ the $i$-th element of the $j$-th column. $\forall \x\in\R^d$, we use $\|\cdot\|$ to denote a general norm in $\R^d$, $\|\cdot\|_*$ its dual norm, $\|\x\|_p$ the $p$-norm of $\x$, defined as $\|\x\|_p=(\sum_{i=1}^d|x_i|^p)^{1/p}$. We use $\|\cdot\|_q$ to denote the dual norm of $p$-norm, where $\frac{1}{p}+\frac{1}{q}=1$. We denote $\B_{\|\cdot\|}$ the $\|\cdot\|$-ball, defined as $\B_{\|\cdot\|}=\{\x\in\R^d|\|\x\|\leq 1\}$. $\forall \x,\x'\in\R^d$, we define the Bregman divergence between $\x$ and $\x'$ with respect to a strictly convex potential function $\Phi(\x)$ as  
$D_{\Phi}(\x,\x')=\Phi(\x)-\Phi(\x')-\nabla \Phi(\x')^{\top}(\x-\x').$ For a positive integer $n$, we denote  $\{1,\dots,n\}$ as $[n]$, and the $(n-1)$-dimensional probability simplex as $\Delta^n$. Let $E:\Delta^n\mapsto\R$ be the negative entropy function, defined as $E(\p)=\sum_{i=1}^np_i\log p_i, \forall \p\in\Delta^n$.\\[-3mm]

\noindent \textbf{Basic setting}\ \ Consider a  set of $n$ data points $\mathcal{S}=\{(\mathbf{x}^{(i)},y^{(i)})\}_{i=1}^n$, where $\mathbf{x}^{(i)}\in\mathbb{R}^d$ is the feature vector for the $i$-th example, and $y^{(i)}\in\{-1,+1\}$ the  corresponding binary label. We are interested the optimization trajectory of first-order methods for minimizing the following unbounded and unregularized empirical risk:
\begin{equation}
\label{eqn:the ERM function}
\textstyle \min\limits_{\w\in\R^d}L(\w)=\frac{1}{n}\sum_{i=1}^nr(\w^{\top}\x^{(i)};y^{(i)}),  
\end{equation}
where $\w\in\R^d$ is a linear classifier, $r:\R\times\{\pm 1\}\mapsto \R$ is the loss function. In this work we focus on the  exponential loss, given by $r(\w^{\top}\x;y)=\exp(-y\x^{\top}\w)$. We introduce the following standard  assumption and definitions.
\begin{definition}[$\|\cdot\|$-margin]
For a linear classifier $\w\in\R^d$ and a norm $\|\cdot\|$, we define its normalized $\|\cdot\|$-margin as
$$\widetilde{\gamma}(\w)=\frac{\min\limits_{i\in[n]}y^{(i)}\w^{\top}\x^{(i)}}{\|\w\|}=\frac{\min\limits_{\p\in\Delta^n}\p^{\top}\A\w}{\|\w\|},$$
where $\A=[\dots;y^{(i)}\x^{(i)\top};\dots]\in\R^{n\times d}$ is the matrix that contains all data.
\end{definition}
\begin{ass}
\label{ass:only}
Assume $\S$ is linearly separable and bounded with respect to some norm $\|\cdot\|$. More specifically, we assume  $\exists \w^*_{\|\cdot\|}\in\B_{\|\cdot\|}$, s.t., 
$\wn=\argmax_{\|\w\|\leq 1}\min_{i\in[n]}y^{(i)}\x^{(i)\top}\w,$ whose margin $\widetilde{\gamma}(\w^*_{\|\cdot\|})=\gamma>0.$ We refer to $\wn$ as the $\|\cdot\|$-maximal margin classifier. Note that, for any $\w\in\R^d$, if $\widetilde{\gamma}(\w)=\gamma$, $\w$ and $\w^*_{\|\cdot\|}$ are at the same direction.  
\end{ass}
\begin{ass}
\label{ass:main:ass}
$\forall i\in [n]$, $\|\x^{(i)}\|_*\leq 1$. That is, the feature vectors are bounded in a unit-ball with respect to the $\|\cdot\|_*$-norm.  
\end{ass}
\begin{definition}[$\|\cdot\|$-Margin maximization rate and $\|\cdot\|$-directional error] Suppose Assumption \ref{ass:only} is satisfied. We consider a sequence of solutions $\w_1,\dots,\w_t,\dots$, and state that $\w_t$ converges to $\w^*_{\|\cdot\|}$ if either $\lim_{t\rightarrow\infty}\widetilde{\gamma}(\w_t)\rightarrow \gamma$, or $\lim_{t\rightarrow \infty}\|\frac{\w_t}{\|\w_t\|}-\w^*_{\|\cdot\|}\|\rightarrow 0$. We define the upper bound on $|\gamma-\widetilde{\gamma}(\w_t)|$ the $\|\cdot\|$-margin maximization rate, and $\|\frac{\w_t}{\|\w_t\|}-\w^*_{\|\cdot\|}\|$ the $\|\cdot\|$-directional error. 

\end{definition}

\section{A Game Framework for Maximizing the Margin}
\label{section:gameframework}

\begin{figure}[t]
\centering
\resizebox{0.9\textwidth}{!}{\begin{minipage}{\textwidth}
\centering
\begin{tikzpicture}[node distance=4.6cm,scale=0.9, every node/.style={scale=0.85}]
  \node (box1) [rectangle, draw]{ \begin{tabular}{cc} 
        \multirow{1}{*} Solving the regularized bilinear game with online learning:  \\      $\max\limits_{\w\in\R^d}\min\limits_{\p\in\Delta^n}\p^{\top}A\w-\Phi(\w) $  \\ 
    \end{tabular} }; %
  \node (box2) [rectangle, draw, below right of=box1] {\begin{tabular}{cc} 
        \multirow{1}{*} Margin maximization rate: $\gamma-\mathcal{O}\left(C_T\right)$ \\ directional error: $\mathcal{O}\left(\sqrt{C_T}\right)$ 
    \end{tabular}};
  \node (box3) [rectangle, draw, below left of=box1] {\begin{tabular}{cc} 
        \multirow{1}{*} Empirical risk minimization \\ with generic optimization methods 
    \end{tabular}};
  
  \draw [->] (box1) -- node[midway,above right] {Plug in regret bounds} (box2);
  \draw [->] (box3) -- node[midway,above left] {Identify equivalent forms} (box1);
\end{tikzpicture}
\caption{Illustration of the game framework for implicit bias analysis. In Section \ref{section:gameframework}, we show that solving a regularized bilinear game with online learning algorithms (top box) can directly maximize the margin, and the convergence rate is on the same order of the averaged regret $C_T$ (right box); In Sections \ref{sec:main result}, we prove that minimizing the empirical risk with a series of generic optimization methods (left box) is equivalent to using online learning algorithms to solve the regularized bilinear game. Thus, the implicit bias rates can be directly obtained by plugging in the regret bounds.}

\label{fig:framework}
\end{minipage}}
\end{figure}

In this section, we present a general game framework and demonstrate that solving this game with online learning algorithms can directly maximize the margin and minimize the directional error. Then, in Section \ref{sec:main result}, we show that many generic optimization methods can be considered to be solving this game with different online dynamics. As a result, the margin maximization rate (and also the directional error) of these optimization methods are exactly characterized by the regret bounds of the corresponding online learning algorithms.
We illustrate this procedure in Figure \ref{fig:framework}. The game objective is defined as follows:
\begin{equation}
\label{eqn:defn:game}    \max_{\w\in\R^d}\min\limits_{\p\in\Delta^n} g(\p,\w)=\p^{\top}\A\w - \Phi(\w),
\end{equation}
where $\Phi(\w)=\frac{1}{2}\|\w\|^2$ is a regularizer and $\|\cdot\|$ denotes some \emph{general norm} in $\R^d$. Inspired by previous work in this vein~\citep{wang2021no,wang2022accelerated}, we apply a weighted no-regret dynamic protocol (summarized in Protocol \ref{pro:no-regret for game}) to solve the game. We first give a brief introduction of Protocol \ref{pro:no-regret for game}, and then present the theorem about the margin of its output. In Protocol \ref{pro:no-regret for game}, the players of the zero-sum game attempt to find the equilibrium by applying online learning algorithms. In each round $t$, the $\p$-player first picks a decision $\p_t$, and passes a weighted loss function to the $\w$-player, defined as $$\alpha_th_t(\w)=-\alpha_t(\p_t^{\top}\A\w-\Phi(\w))=-\alpha_t g(\p_t,\w).$$ Then, the $\w$-player observes the loss, picks a decision $\w_t$, and passes a weighted loss function 
$$\alpha_t\ell_t(\p)=\alpha_t(\p^{\top}\A\w_t-\Phi(\w_t))=\alpha_tg(\p,\w_t),$$ 
to the $\p$-player. (Note that the order of the two players can also be reversed.) After $T$ iterations, the algorithm outputs the weighted sum of the $\w$-player's decisions: $\widetilde{\w}_T=\sum_{t=1}^T\alpha_t\w_t$. Under this framework, we define the weighted regret upper bound of both players respectively as 
\begin{align*}
   \sum_{t=1}^T \alpha_t\ell_t(\p_t)-\min\limits_{\p\in\Delta^n}  \sum_{t=1}^T \alpha_t\ell_t(\p)\leq   \text{Reg}^{\p}_T,\  
   \sum_{t=1}^T \alpha_th_t(\w_t)-\min\limits_{\w\in\R^d}  \sum_{t=1}^T \alpha_th_t(\w) \leq  \text{Reg}^{\w}_T.
\end{align*}
Further, we denote the upper bound on the \emph{average} weighted regret by $C_T=(\text{Reg}_T^{\p}+\text{Reg}_T^{\w})/\sum_{t=1}^T\alpha_t$. We have the following conclusion on the margin and directional error of $\widetilde{\w}_T$, which is proved in Section~\ref{appendix:Theorem 1}.

\begin{figure}[t]
\centering
\begin{minipage}{0.9\textwidth}
\begin{protocoll}[H]
\resizebox{0.99\linewidth}{!}{
\caption{No-regret dynamics with weighted OCO for solving $g(\p,\w)$}
\begin{minipage}{0.99\textwidth}
\begin{algorithmic}[1]
\label{pro:no-regret for game}
\STATE \textbf{Initialization}: $\textsf{OL}^{\w}$, $\textsf{OL}^{\p}$. //  The online algorithms for choosing $\w$ and $\p$.
\FOR{$t=1,\dots,T$}
\STATE $\w_t\leftarrow \textsf{OL}^{\w}$;
\STATE $\textsf{OL}^{\p}\leftarrow\alpha_t,\ell_t(\cdot)$; // Define $\ell_t(\cdot)=g(\w_t,\cdot)$
\STATE $\p_t\leftarrow\textsf{OL}^{\p}$;
\STATE $\textsf{OL}^{\w}\leftarrow\alpha_t,h_t(\cdot)$; // Define $h_t(\cdot)=-g(\cdot,\p_t)$
\ENDFOR
\STATE \textbf{Output}: $\widetilde{\w}_T=\sum_{t=1}^{T}\alpha_t\w_t.$
\end{algorithmic}
\end{minipage}
}
\end{protocoll}
\end{minipage}
\end{figure}

\begin{theorem}
\label{thm:margin}
Suppose Assumption \ref{ass:only} holds with respect to some general norm $\|\cdot\|$. Consider solving the two-player zero-sum game defined in  \eqref{eqn:defn:game} by applying Protocol \ref{pro:no-regret for game}.
Then $\widetilde{\w}_T$ will have a positive margin on round $T$ if $C_T\leq \frac{\gamma^2}{4}$. Moreover, as long as $C_T\leq \frac{\gamma^2}{4}$, we have   
\begin{equation}
\label{eqn:normalized margin}
\frac{\min\limits_{\p\in\Delta^n}\p^{\top}\A\widetilde{\w}_T}{\left\|\widetilde{\w}_T\right\|} \geq \gamma - \frac{4C_T}{\gamma^2}.
\end{equation}
If $\Phi(\w)$ is $\lambda$-strongly convex with respect to the norm $\|\cdot\|$, we have 
$$\left\| \frac{\widetilde{\w}_T}{\|\widetilde{\w}_T\|}-\w_{\|\cdot\|}^*\right\|\leq \frac{8\sqrt{2}}{\gamma^2\sqrt{\lambda}}\sqrt{C_T}.$$
\end{theorem}

Theorem \ref{thm:margin} shows that the output of Protocol \ref{pro:no-regret for game}, denoted as $\widetilde{\w}_T$, achieves a positive margin when the average regret $C_T\leq \frac{\gamma^2}{4}$. In the following sections, we demonstrate that with appropriately chosen online learning algorithms $C_T$ always decreases with respect to $T$; in fact $C_T \to 0$ as $T \to \infty$. Therefore, once the condition $C_T\leq \frac{\gamma^2}{4}$ is met for a particular value $T_0$, it will also be met for all $T \geq T_0$. Thereafter, $\widetilde{\w}_T$ continues to increase the $\|\cdot\|$-margin and converges to the maximum $\|\cdot\|$-margin classifier, and the rate is directly characterized by $C_T$.  Since $C_T$ is the average regret of the online learning algorithms, better bounds on $C_T$ lead to a less stringent condition on large enough $T.$ Finally, we note that the condition on sufficiently large $T$ is also (explicitly or implicitly) required in all previous work on the non-asymptotic margin maximization rates of generic methods \citep{nacson2019convergence,li2021implicit,sun2022mirror}.  
\subsection{Proof of Theorem \ref{thm:margin}}
\label{appendix:Theorem 1}

Define $m(\w)=\min_{\p\in\Delta^n} g(\p,\w)$, $\overline{\w}_T=\frac{1}{\sum_{t=1}^T\alpha_t}\sum_{t=1}^T\alpha_t\w_t=\frac{1}{\sum_{t=1}^T\alpha_t}{\widetilde{\w}_T}$. We introduce the following lemma, which shows that using online learning for solving the game defined in \eqref{eqn:defn:game} maximizes $m(\w)$.  

\begin{lemma}[\citet{abernethy2018faster}]
\label{lem:margin:gap}
Consider solving the game defined in \eqref{eqn:defn:game} with the online learning dynamic defined in Protocol \ref{pro:no-regret for game}. 
We have, for all $\w\in\R^d$,
$m(\w)-m(\overline{\w}_T)\leq \frac{\emph{Reg}_T^{\p}+\emph{Reg}_T^{\w}}{\sum_{t=1}^T\alpha_t}.$
\end{lemma}

\noindent Based on Lemma \ref{lem:margin:gap} and the definition of $m(\cdot)$, we can write
\begin{align*}
  m\left(\overline{\w}_T\right) =  m\left(\frac{\widetilde{\w}_T}{\sum_{t=1}^T\alpha_t}\right) = {} & \min\limits_{\p\in\Delta^n}\p^{\top}\A\frac{\widetilde{\w}_T}{\sum_{t=1}^T\alpha_t} -\frac{1}{2}\left\|\frac{\widetilde{\w}_T}{\sum_{t=1}^T\alpha_t}\right\|^2\notag\\
 \geq {} & m\left(\left\|\frac{\widetilde{\w}_T}{\sum_{t=1}^T\alpha_t}\right\|\w^*_{\|\cdot\|}\right) -\frac{\text{Reg}_T^{\p}+\text{Reg}_T^{\w}}{\sum_{t=1}^T\alpha_t}\notag \\
 = {} & \gamma\left\|\frac{\widetilde{\w}_T}{\sum_{t=1}^T\alpha_t}\right\| - \frac{1}{2}\left\|\frac{\widetilde{\w}_T}{\sum_{t=1}^T\alpha_t}\right\|^2 -\frac{\text{Reg}_T^{\p}+\text{Reg}_T^{\w}}{\sum_{t=1}^T\alpha_t},
\end{align*}
which implies that 
\begin{equation}
\label{eqn:proof:main:1}
\frac{\min\limits_{\p\in\Delta^n}\p^{\top}\A\widetilde{\w}_T}{\left\|\widetilde{\w}_T\right\|} \geq \gamma - \frac{\text{Reg}_T^{\p}+\text{Reg}_T^{\w}}{\|\widetilde{\w}_T\|}~.   
\end{equation}
The above  proof follows the main idea in \citet{wang2022accelerated}. Next, we turn to lower bound $\|\widetilde{\w}_T\|$. Note that since ${\w}_T$ (and therefore $\widetilde{\w}_T$) does not have a simple explicit form, the the technique for lower bounding the norm in \citet{wang2022accelerated} fails and 
we need to find a new approach. Let $(\x,y)\in\{(\x^{(i)},y^{(i)})\}_{i=1}^n$ be a data point. We have
\begin{equation}
\label{eqn:proof:main:2}
 \|\widetilde{\w}_T\|\geq  \|y\x\|_*\|\widetilde{\w}_T\| \geq y\x^{\top}\widetilde{\w}_T\geq \min\limits_{\p\in\Delta^n}\p^{\top}\A\widetilde{\w}_T,
\end{equation}
where the first inequality is due to assumption that $\|\x\|_*\leq 1$, and the second inequality is derived from the Cauchy-Schwarz inequality. To proceed, we need a lower bound on the unnormalized margin of $\widetilde{\w}_T$. We have 
\begin{align}
\label{eqn:lower bound of bar wT}
  m\left(\overline{\w}_T\right) =  m\left(\frac{\widetilde{\w}_T}{\sum_{t=1}^T\alpha_t}\right) = {} & \min\limits_{\p\in\Delta^n}\p^{\top}\A\frac{\widetilde{\w}_T}{\sum_{t=1}^T\alpha_t} -\frac{1}{2}\left\|\frac{\widetilde{\w}_T}{\sum_{t=1}^T\alpha_t}\right\|^2\notag\\
 \geq {} & m\left(\gamma\w_{\|\cdot\|}^*\right) -\frac{\text{Reg}_T^{\p}+\text{Reg}_T^{\w}}{\sum_{t=1}^T\alpha_t}\notag \\
 = {} & \min\limits_{\p\in\Delta^n}\p^{\top}\A\w_{\|\cdot\|}^* -\frac{1}{2}\|\gamma\w_{\|\cdot\|}^*\|^2 - \frac{\text{Reg}_T^{\p}+\text{Reg}_T^{\w}}{\sum_{t=1}^T\alpha_t} \notag\\
 = {} & \frac{\gamma^2}{2}-\frac{\text{Reg}_T^{\p}+\text{Reg}_T^{\w}}{\sum_{t=1}^T\alpha_t},
\end{align}
where for the first inequality we apply Lemma \ref{lem:margin:gap} and compare $m(\widetilde{\w}_T)$ with that of $\gamma \w^*_{\|\cdot\|}$, and the last equality is derived based on Assumption \ref{ass:only} (the margin of $\w^*_{\|\cdot\|}$ is $\gamma$, and $\|\w^*_{\|\cdot\|}\|=1$). \eqref{eqn:lower bound of bar wT} suggests that
\begin{equation}
\label{eqn:proof:main:3}
\min\limits_{\p\in\Delta^n}\p^{\top}\A\widetilde{\w}_T\geq \underbrace{\frac{1}{2}\frac{\left\|\widetilde{\w}_T\right\|^2}{\sum_{t=1}^T\alpha_t}}_{\geq 0}+
\frac{\gamma^2}{2}\sum_{t=1}^T\alpha_t - \left(\text{Reg}_T^{\p}+\text{Reg}_T^{\w}\right) \geq \frac{\gamma^2}{2}\sum_{t=1}^T\alpha_t - \left(\text{Reg}_T^{\p}+\text{Reg}_T^{\w}\right).  
\end{equation}
Note that, to plug in the lower bound of $\widetilde{\w}_T$, we need to ensure the RHS of \eqref{eqn:proof:main:3} is positive.  When $\frac{\gamma^2}{2}\sum_{t=1}^T\alpha_t\geq 2\left(\text{Reg}_T^{\p}+\text{Reg}_T^{\w}\right)$, we have
$$\|\widetilde{\w}_T\|\geq \frac{\gamma^2}{4}\sum_{t=1}^T\alpha_t+\left[\frac{\gamma^2}{4}\sum_{t=1}^T\alpha_t-\left(\text{Reg}_T^{\p}+\text{Reg}_T^{\w}\right)\right]\geq \frac{\gamma^2}{4}\sum_{t=1}^T\alpha_t.$$
Combining \eqref{eqn:proof:main:1}, \eqref{eqn:proof:main:2}, and \eqref{eqn:proof:main:3}, we have 
$
\frac{\min\limits_{\p\in\Delta^n}\p^{\top}\A\widetilde{\w}_T}{\left\|\widetilde{\w}_T\right\|} \geq \gamma - \frac{4\left(\text{Reg}_T^{\p}+\text{Reg}_T^{\w}\right)}{\gamma^2\sum_{t=1}^T\alpha_t}=\frac{4C_T}{\gamma^2}.  
$
Note that to apply \eqref{eqn:proof:main:3}, we need $\frac{\gamma^2}{2}\sum_{t=1}^T\alpha_t - 2\left(\text{Reg}_T^{\p}+\text{Reg}_T^{\w}\right)\geq 0$.\\

\noindent Finally, we focus on the distance between $\frac{\widetilde{\w}_T}{\|\widetilde{\w}_T\|}$ and $\w^*_{\|\cdot\|}$ for the case where $\Phi(\w)$ is strongly convex with respect to $\|\cdot\|$. This part of the proof is motivated by Theorem 4 of \citet{ramdas2016towards}, who show that a variant of the perceptron algorithm can converge to the $\ell_2$-maximum margin classifier in an $\mathcal{O}(1/\sqrt{t})$ convergence rate.  We have 
\begin{align}
\left\|\frac{\widetilde{\w}_T}{\|\widetilde{\w}_T\|}-\w_{\|\cdot\|}^* \right\| =  & \left\|\frac{\overline{\w}_T}{\|\overline{\w}_T\|}-\w_{\|\cdot\|}^* \right\| =  \frac{\|\overline{\w}_T-\|\overline{\w}_T\|\w_{\|\cdot\|}^*\|}{\|\overline{\w}_T\|} \notag\\
= {} & \frac{\|\overline{\w}_T-\gamma \w_{\|\cdot\|}^*+\gamma \w_{\|\cdot\|}^*-\|\overline{\w}_T\|\w_{\|\cdot\|}^*\|}{\|\w\|}\notag\\
 \leq {} & \frac{\|\overline{\w}_T-\gamma \w_{\|\cdot\|}^*\| + |\gamma-\|\overline{\w}_T\||}{\|\overline{\w}_T\|}\notag\\
= {} & \frac{\|\overline{\w}_T-\gamma \w_{\|\cdot\|}^*\| + |\|\gamma\w^*_{\|\cdot\|}\|-\|\overline{\w}_T\||}{\|\overline{\w}_T\|}
\leq {} \frac{2\|\overline{\w}_T-\gamma\w_{\|\cdot\|}^*\|}{\|\overline{\w}_T\|}~,\label{eqn:appb:11111}
\end{align}
where the first inequality is based on the Minkowski inequality and the fact that $\|\w^*_{\|\cdot\|}\|=1$. Next, note that $m(\w)$ is $\lambda$-strongly concave with respect to $\|\cdot\|$,  and $\gamma \w^*_{\|\cdot\|}$ maximize $m(\w)$. This is because it is easy to see that the optimal solution of $m(\w)$ always lies in the direction of $\w^*_{\|\cdot\|}$, and we only need to decide the norm. Let $c>0$ be some constant, we have $m\left(c\w^*_{\|\cdot\|}\right)=c\gamma -\frac{1}{2}c^2$. The function is maximized when $c=\gamma$, which implies that the optimal solution is $\gamma\w^*_{\|\cdot\|}$. Combining these facts with Lemma \ref{lem:margin:gap}, we have 
\begin{equation}
\label{eqn:appb:22222222}
  \frac{\lambda}{2}\|\overline{\w}_T-\gamma\w^*_{\|\cdot\|}\|^2\leq m(\gamma \w^*_{\|\cdot\|})-m(\overline{\w}_T)\leq \frac{\text{Reg}_T^{\p}+\text{Reg}_T^{\w}}{\sum_{t=1}^T\alpha_t}.  
\end{equation}
Finally, combining with \eqref{eqn:proof:main:2}, we have when $\frac{\gamma^2}{2}\sum_{t=1}^T\alpha_t-2\left(\text{Reg}_T^{\p}+\text{Reg}_T^{\w}\right)\geq 0$,
$$
\|\overline{\w}_T\|=\frac{1}{\sum_{t=1}^T\alpha_t}\|\widetilde{\w}_T\|\geq \frac{1}{\sum_{t=1}^T\alpha_t}\left[ \frac{\gamma^2}{4}\sum_{t=1}^T\alpha_t\right]=\frac{\gamma^2}{4}~.
$$
Combining with \eqref{eqn:appb:11111} and \eqref{eqn:appb:22222222}, we obtain
$\left\|\frac{\widetilde{\w}_T}{\|\widetilde{\w}_T\|}-\w^*_{\|\cdot\|} \right\|\leq  \frac{8\sqrt{2(\text{Reg}^{\p}_T+\text{Reg}_T^{\w}})}{\gamma^2\sqrt{\lambda\sum_{t=1}^T\alpha_t}}$. 

\section{Implicit Bias of Generic Methods}
\label{sec:main result}

In this section, we show that average mirror descent and steepest descent can find their equivalent online learning forms under Protocol \ref{pro:no-regret for game}. Thus, their margin maximization rates can be directly characterized by the corresponding average regret $C_T$. For clarity, we use $\v_t$ to denote the classifier updates in the original methods, and $\w_t$ the update under the game framework. Note that Theorem \ref{thm:margin} clearly implies that the convergence rate of the directional error is always a square-root worse than that of the margin maximization rate. Thus, we only present the margin maximization rates here; the corresponding rates on directional error are presented in Section \ref{section:detailed proof}.  
\subsection{Mirror-Descent-Type of Methods }
\label{sec:mirror descent}
First, we consider minimizing \eqref{eqn:the ERM function} by applying the following mirror descent algorithm:
\begin{equation}  
\label{eqn:mirror:descent:defn}
\v_t=\argmin\limits_{\v\in\R^d}\eta_t\left\langle{\nabla L(\v_{t-1})},\v \right\rangle + D_{\Phi}(\v,\v_{t-1}),
\end{equation}
where $D_{\Phi}(\v,\v_{t-1})$ is the Bregman divergence between $\v$ and $\v_{t-1}$, and $\Phi(\v)$ is a strongly convex potential function that defines the mirror map. 
Note that since the feasible domain in \eqref{eqn:mirror:descent:defn} is unbounded, we can rewrite the algorithm in the following form $\nabla \Phi(\v_t)=\nabla \Phi(\v_{t-1})-\eta_t\nabla L(\v_{t-1}).$

In this paper, we consider weighted-average mirror descent with the squared $q$-norm, i.e., $\Phi(\w)=\frac{1}{2}\|\w\|_q^2$, where $q\in(1,2]$, and demonstrate that this optimization algorithm can enable faster $\|\cdot\|$-margin maximization rates. 
The detailed update rule is summarized in the left box of Algorithm \ref{alg:mirror descent}. It is worth noting that this type of regularizer is $(q-1)$-strongly convex with respect to $\|\cdot\|_q$, and can be updated efficiently in closed form as below\footnote{This expression appears in~\citep[Section 6.7,][]{orabona2019modern} and we reproduce it for completeness.}: for each coordinate $i\in[d]$, we have
\begin{equation}
\begin{split}
\label{eqn:efffmirror}
 \widehat{v}_{t,i} = {} & \text{sign}(v_{t-1,i})|v_{t-1,i}|^{q-1}\|\v_{t-1}\|_q^{2-q} - \eta_t{[\nabla L(\v_{t-1})]_i},\\
 v_{t,i} = {} & \text{sign}( \widehat{v}_{t,i})| \widehat{v}_{t,i}|^{p-1}\| \widehat{\v}_{t}\|_p^{2-p}.
\end{split}
\end{equation}
We make a few final observations about this algorithm: 1) Instead of using the weighted sum $\widetilde{\v}_T$, we could output the weighted average $\frac{\widetilde{\v}_t}{\sum_{s=1}^t \alpha_s}$ without altering the margin or directional convergence rate. This is attributed to the scale-invariance of the margin, i.e., $\forall c>0, \w\in\R^d$, $\widetilde{\gamma}(\w)=\widetilde{\gamma}(c\w)$. The same argument applies to the directional error. 2) The use of the weighted average is standard in the analysis of mirror descent (e.g., Section 4.2 of \citet{bubeck2015convex}). This paper shows that using non-uniform weights is advantageous for achieving rapid margin maximization rates; 3) The per-round computational complexity of \eqref{eqn:efffmirror} is $\mathcal{O}(d)$, which is similar to the $p$-mirror-descent variant in~\cite{sun2022mirror}.

For Algorithm \ref{alg:mirror descent}, we have the following theorem. We defer its proof in Section \ref{appendix:Proof of section 3}, along with a more general theorem that allows a general configuration of the parameters $\eta_t$, $\alpha_t$ and $\beta_t$.

\begin{theorem}
\label{thm:mirror margin}
Suppose Assumption \ref{ass:only} holds wrt $\|\cdot\|_q$-norm for $q\in(1,2]$. For the left box of Algorithm \ref{alg:mirror descent}, let $\eta_t=\frac{1}{L(\v_{t-1})}$. For the right box, let $\alpha_t=1$, and $\beta_1=1$, $\beta_{t}=\frac{1}{t-1}$. Then the methods in the two boxes of Algorithm \ref{alg:mirror descent} are identical, in the sense that $\widetilde{\v}_T=\widetilde{\w}_T$. Moreover, we have the average regret upper bound 
$C_T = \frac{\left(\frac{2}{q-1}+2\log n\right)(\log T+2)}{T}.$ Therefore, the algorithm achieves a positive margin  when $T$ is sufficiently large such that $ T\geq \frac{4 \left(\frac{2}{q-1}+2\log n\right)(\log T+2)}{\gamma^2}$. We have the convergence rate
\begin{equation}
    \label{eqn:thorem1:mirror}    \frac{\min\limits_{\p\in\Delta^n}\p^{\top}\A\widetilde{\v}_T}{\left\|\widetilde{\v}_T\right\|_q} \geq \gamma -  \frac{4 \left(\frac{2}{q-1}+2\log n\right)(\log T+2)}{\gamma^2T}=\gamma - {O}\left(\frac{\log n \log T}{(q-1)\gamma^2T}\right).
\end{equation}
\end{theorem}

\begin{figure}
\centering
\begin{minipage}{0.9\textwidth}
\begin{algorithm}[H]
\scalebox{0.91}{
\caption{Mirror Descent [Recall $\ell_t(\p)=g(\p,\w_t),$ and $h_t(\w)=-g(\p_t,\w)$] }
\fbox{
  \begin{minipage}[t]{0.45\textwidth}
  \begin{algorithmic}[1]
  
\vspace{0.2cm}


\FOR{$t=1,\dots,T$}

\vspace{0.04cm}

\STATE \hspace{-0.3cm}$\nabla \Phi(\v_t)=\nabla\Phi(\v_{t-1}) - \eta_t \nabla L(\v_{t-1})$

\ENDFOR

\vspace{0.03cm}
\STATE \textbf{Output}: $\widetilde{\v}_{T}=\sum_{t=1}^T\frac{1}{t}\v_t$  

\vspace{0.15cm}
  \end{algorithmic}
  \end{minipage}
}
\hfill
\fbox{
  \begin{minipage}[t]{0.55\textwidth}
  \begin{algorithmic}
  \STATE \hspace{-0.25cm}$\p$-player: $\p_t=  \argmin\limits_{\p\in\Delta^n} \alpha_t\ell_{t-1}(\p)+\beta_tD_{E}\left(\p,\frac{\mathbf{1}}{n}\right)$
  \STATE \hspace{-0.25cm}$\w$-player: $\w_t =  \argmin\limits_{\w\in\R^d} \sum_{j=1}^{t}\alpha_j h_j(\w) $
  \vspace{0.1cm}
  \STATE \hspace{-0.25cm}{\bf Output:} $\widetilde{\w}_T=\sum_{t=1}^T\alpha_t\w_t$
  \vspace{0.05cm}
  \end{algorithmic}
  \end{minipage}
}
\label{alg:mirror descent}
}
\end{algorithm}
\end{minipage}
\end{figure}

The first part of Theorem \ref{thm:mirror margin} indicates that the mirror descent algorithm can be described as two players using certain cleverly designed online learning algorithms to solve the \emph{regularized} bilinear game in \eqref{eqn:defn:game}. For the $\p$-player, we propose a new and unusual online learning algorithm, which we call \emph{regularized greedy}, defined as:
$$\textstyle \p_t=  \argmin_{\p\in\Delta^n} \alpha_t\ell_{t-1}(\p)+\beta_tD_{E}\left(\p,\frac{\mathbf{1}}{n} \right),$$
Essentially, in round $t$, the $\p$-player minimizes the previous round's loss function, $\ell_{t-1}$, plus a regularizer at round $t$, and the two terms are balanced by the parameters $\alpha_t$ and ${\beta_t}$. On the other hand, 
we select the \emph{follow-the-leader$^+$} algorithm for the $\w$-player: 
$$ \textstyle\w_t =  \argmin_{\w\in\R^d} \sum_{j=1}^{t}\alpha_jh_j(\w),$$
which returns the solution that minimize the cumulative loss so far. The $+$ sign in the name is because the algorithm can pick the decision $\w_t$ \emph{after} seeing its loss function. 
This is an interesting and unusual design because the regularized greedy algorithm will clearly suffer a worst-case \emph{linear} regret for the $\p$-player. Luckily, we find that for our specific problem, the dominating term of the the $\p$-player's regret bound can be canceled by the $\w$-player's regret bound, which is negative as the corresponding algorithm used is \emph{clairvoyant}, i.e.~can see the current loss $\ell_t$ before making a decision at round $t$.
This ensures that sublinear (and more generally fast) rates are possible. Note that $\beta_t$ and $\alpha_t$ will influence both the regret bound and the algorithm equivalence analysis, so finding the right parameter configuration that works for both is a non-trivial task. We make the choice $\beta_t = \frac{\alpha_t}{\sum_{i=1}^{t-1}\alpha_i}$, which ensures both algorithmic equivalence and sublinear regret bounds.

The second part of Theorem \ref{thm:mirror margin} shows that the average regret $C_T$ of Algorithm \ref{alg:mirror descent} is on the order of $\mathcal{O}\left(\frac{\log n\log T}{(q-1)\gamma^2T}\right)$. Therefore, by plugging in Theorem \ref{thm:margin}, we observe that the margin shrinks on the order of $\textstyle \gamma-\mathcal{O}\left(\frac{\log n\log T}{\gamma^2(q-1)T}\right)$, and the implicit bias convergence rate is $\mathcal{O}\left(\frac{\log n\log T}{\gamma^2(q-1)\sqrt{T}}\right).$ Next, we show an improved rate with a more aggressive step size of order  $\mathcal{O}\left(\frac{t}{L(\v_t)}\right)$ instead of  $\mathcal{O}\left(\frac{1}{L(\v_t)}\right)$. The proof is given in Section \ref{appendix:Proof of section 3}.
\begin{theorem}
\label{cor:u2222}
Suppose Assumption \ref{ass:only} holds wrt $\|\cdot\|_q$-norm for $q\in(1,2]$. For the left box of Algorithm \ref{alg:mirror descent}, let $\eta_t=\frac{t}{L(\v_{t-1})}$, and let the final output be $\widetilde{\v}_{T}=\sum_{t=1}^T\frac{2}{t+1}\v_t$. For the right box, let $\alpha_t=t$, and $\beta_1=1$, $\beta_{t}=\frac{2}{t-1}$. Then the two algorithms are identical, in the sense that $\widetilde{\v}_T=\widetilde{\w}_T$. Moreover, when $T\geq  \sqrt{\frac{8\left[\frac{4T}{q-1}  + 4\log n \log T+1+2\log n\right]}{\gamma^2}}$, we have 
\begin{equation}    \frac{\min\limits_{\p\in\Delta^n}\p^{\top}\A\widetilde{\v}_T}{\left\|\widetilde{\v}_T\right\|_q} \geq \gamma - \frac{32}{\gamma^2T(q-1)}
-\frac{8(4\log n \log T+1+2\log n)}{\gamma^2T^2}~.
\end{equation}
\end{theorem}
Observe that the margin maximization rate in Theorem \ref{cor:u2222} is $\mathcal{O}\left(\frac{1}{(q-1)\gamma^2T}\right)+\mathcal{O}\left(\frac{\log n\log T}{\gamma^2T^2}\right)$. Compared to \eqref{eqn:thorem1:mirror}, it has a better dependence on $\log n$ and $\log T$.

\begin{figure}
\centering
\begin{minipage}{0.9\textwidth}
\begin{algorithm}[H]
\scalebox{0.91}{
\caption{ Momentum-based MD }
\fbox{
  \begin{minipage}[t]{0.47\textwidth}
  \begin{algorithmic}[1]
  \label{alg:momentum mirror}
\vspace{0.1cm}
\FOR{$t=1,\dots,T$}

\STATE 
 \hspace{-0.1cm}$\nabla \Phi(\v_t) 
    =\nabla \Phi(\v_{t-1}) - \eta_t  {\nabla L({\v}_{t-1})}$\\
    \vspace{0.08cm}
  \hspace{-0.1cm}  $- \left( \widehat{\eta}_t{\nabla L({\v}_{t-1})} - {\eta}_{t-1}{\nabla L({\v}_{t-2})}\right)$

\ENDFOR
\STATE \textbf{Output}: $\widetilde{\v}_{T}=\sum_{t=1}^T\frac{1}{t}\v_t$
  \end{algorithmic}
  \end{minipage}
}
\hfill
\fbox{
  \begin{minipage}[t]{0.53\textwidth}
  \begin{algorithmic}
  \vspace{0.05cm}
  \STATE \hspace{-0.25cm}$\w$-player: \\
\hspace{-0.25cm}$\w_t=\argmin_{\w\in\R^d}\sum_{i=1}^{t-1}\alpha_i h_i(\w)+\alpha_th_{t-1}(\w)$
  
  \STATE \hspace{-0.25cm}$\p$-player: \\
\hspace{-0.25cm}{\centering $\p_t=\argmin_{\p\in\Delta^n}
  \alpha_t\ell_t(\p)+\beta_t D_{E}(\p,\p_0)$}
  \STATE 
  \hspace{-0.25cm}{\bf Output:} $\widetilde{\w}_T=\sum_{t=1}^T\alpha_t\w_t$
    \vspace{0.05cm}
  \end{algorithmic}
  \end{minipage}
}
}
\end{algorithm}
\end{minipage}
\end{figure}

Finally, we focus on a momentum-based MD method, which is given in Algorithm \ref{alg:momentum mirror}. For this algorithm, we have the following guarantee, which is proved in Section \ref{section:thm:mirror momentum}.
\begin{theorem}
\label{thm:mirror momentum}
Suppose Assumption \ref{ass:only} holds wrt $\|\cdot\|_q$-norm for $q\in(1,2]$. For the left box of Algorithm \ref{alg:momentum mirror}, let $\eta_t=\frac{t}{L(\v_{t-1})}$, and $\widehat{\eta}_{t}=\frac{t-1}{L(\v_{t-1})}$. For the second box, let $\alpha_t=t$, and $\beta_t=\frac{2}{t+1}$. Then the methods in the two boxes of Algorithm \ref{alg:momentum mirror} are identical, in the sense that $\widetilde{\v}_T=\widetilde{\w}_T$. Moreover, when $T\geq \frac{\sqrt{8\left(4\log n \log T +\frac{2T}{q-1} \right)}}{\gamma}$, we have 
\begin{equation}
    \label{eqn:thorem4:mirror}    \frac{\min\limits_{\p\in\Delta^n}\p^{\top}\A\widetilde{\v}_T}{\left\|\widetilde{\v}_T\right\|_q} \geq \gamma - \frac{\sum_{t=1}^T\|\p_t-\p_{t-1}\|^2_1}{\gamma^2(q-1)T^2}-\frac{32\log n\log T}{\gamma^2T^2}~. 
\end{equation}
\end{theorem}
The above theorem shows that, for sufficiently large $T$, the margin maximization rate can be data-dependent. Note that $\sum_{t=1}^T\|\p_{t}-\p_{t-1}\|^2_1\leq 2T$, so in the worst case, the bound reduces to the results in Theorem \ref{cor:u2222}, but it can become significantly better when  $\sum_{t=1}^T\|\p_{t}-\p_{t-1}\|^2_1$ is small. We expect that when $T$ is very large, $\p_T$ will change very slowly as we already know that the direction of $\widetilde{\v}_T$ will converge --- however, turning this into a precise faster rate dependent on the original training data geometry, i.e.~$\A$, is an intriguing open question. 

\subsection{Steepest Descent}
\label{section:SD}

\begin{figure}
\centering
\begin{minipage}{0.9\textwidth}
\begin{algorithm}[H]
\scalebox{0.91}{
\caption{Steepest Descent [Recall $\ell_t(\p)=g(\p,\w_t),$ and $h_t(\w)=-g(\p_t,\w)$]}
\fbox{
  \begin{minipage}[t]{0.39\textwidth}
  \begin{algorithmic}[1]
  \label{alg:steepet descent:main}
\vspace{0.1cm}
\FOR{$t=1,\dots,T$}
\STATE $\s_{t-1} = \argmin\limits_{\|\s\|\leq 1} \s^{\top}\nabla L(\v_{t-1})$
\STATE $\v_{t}=\v_{t-1} + \eta_{t-1} \s_{t-1}$
\ENDFOR
\STATE \textbf{Output}: $\v_{T}$
  \end{algorithmic}
  \end{minipage}
}
\hfill
\fbox{%
     \begin{minipage}[t]{0.61\textwidth}
  \begin{algorithmic}
    \STATE $\w$-player: 
$\w_t =  \argmin\limits_{\w\in\R^d} \left\langle \delta_{t-1}\alpha_{t-1}\nabla h_{t-1}(\w_{t-1}),\w\right\rangle$\\
\hspace{3.3cm}$+ D_{\frac{1}{2}\|\cdot\|^2}(\w,\w_{t-1})$
  \STATE $\p$-player: $ \p_t=   \argmin\limits_{\p\in\Delta^n}\sum_{i=1}^{t}\alpha_i\ell_{i}(\p)+D_{E}\left(\p,\frac{\mathbf{1}}{n}\right)$
  \STATE 
 {\bf Output:} $\widetilde{\w}_T=\sum_{t=1}^T\alpha_t{\w}_t.$
 
  \end{algorithmic}
  \end{minipage}
}

}
\end{algorithm}

\end{minipage}
\end{figure}

Next, we consider the steepest descent method under a \emph{general} norm $\|\cdot\|$. For a succinct description of this algorithm see~\citet{Convex-Optimization}. For completeness, we have also described this algorithm in the left box of Algorithm \ref{alg:steepet descent:main}. In each iteration $t$, the algorithm first identifies the steepest direction with respect to the norm $\|\cdot\|$ (Step 2). It then adjusts the decision towards this direction using a specific step size $\eta_t$ (Step 3). After $T$ iterations, the algorithm yields the final iteration $\v_T$. In the following, we show that an $\mathcal{O}\left(\frac{\lambda +\log n}{\gamma^2\lambda T}\right)$ $\|\cdot\|$-margin maximization rate can be achieved when the squared-norm, i.e. $\frac{1}{2}\|\cdot\|^2$ is $\lambda$-strongly convex (e.g., $\frac{1}{2}\|\cdot\|_q^2$ is $(q-1)$-strongly convex wrt $\|\cdot\|_q$). The proof of this result provided in Section \ref{appendix:sddddd}. Moreover, we note that the slower $\mathcal{O}\left(\frac{\log n+\log T}{\sqrt{T}}\right)$ of \citet{nacson2019convergence} rate can be recovered as a special case for norms that are not necessarily strongly convex.
\begin{theorem}
\label{thm:steep:acc::main}
Suppose Assumption \ref{ass:only} holds wrt a general norm $\|\cdot\|$, and $\frac{1}{2}\|\cdot\|^2$ is $\lambda$-strongly convex wrt $\|\cdot\|$. Let $\eta_t = \frac{\alpha_t \|\nabla L(\w_t)\|}{L(\w_t)}$, and $\delta_{t-1}=\alpha_{t-1}$. Then the methods in the two boxes of Algorithm \ref{alg:steepet descent:main}  are are equivalent, in the sense that $\v_T=\widetilde{\w}_T$. Moreover, let $\alpha_t=\frac{\lambda}{2}$. Then 
$C_T= \frac{\frac{\lambda}{4}+\log n }{T\lambda}.$
Therefore, when $T\geq\frac{\lambda + 4\log n}{\lambda\gamma^2}$, we have 
$$\frac{\min\limits_{\p\in\Delta^n}\p^{\top}\A{\v_T}}{\|\v_T\|}\geq \gamma - \frac{\lambda + 4\log n}{\gamma^2T\lambda}.$$
\end{theorem}
The first part of Theorem \ref{thm:steep:acc::main} elucidates the equivalent online dynamic of the steepest descent algorithm, which is also depicted in the right box of Algorithm \ref{alg:steepet descent:main}. The $\w$-player employs the standard online mirror descent (OMD) algorithm \citep{Intro:Online:Convex}, while the $\p$-player utilizes FTRL$^+$, i.e., $\p_t$ is selected by minimizing the cumulative loss observed so far, coupled with a regularization term. The crux of our algorithm equivalence analysis lies in evaluating the output of the $\w$-player. For this, we initially prove that given $\delta_{t}=\frac{1}{\alpha_t}$, the OMD algorithm condenses to best-response (BR), that is, $\w_t=\argmin_{\w\in\R^d}\alpha_th_{t-1}(\w)$. We then prove that BR's output coincides with the steepest descent direction. The second part of Theorem  \ref{thm:steep:acc::main} shows that the average regret of this online learning dynamic is $\mathcal{O}\left(\frac{\lambda +\log n}{\gamma^2\lambda T}\right)$, which leads to the corresponding fast margin maximization/small direction error. We note that the favorable average regret is made possible by allowing the two players to play against each other, rather than plugging in worst-case regret bounds.
\subsection{Even Faster Rates with Accelerated Generic Methods}

\begin{figure}
\centering
\begin{minipage}{0.9\textwidth}
\begin{algorithm}[H]
{
\caption{Accelerated Methods [Recall $\ell_t(\p)=g(\p,\w_t),$ and $h_t(\w)=-g(\p_t,\w)$]}
\fbox{
  \begin{minipage}[t]{0.4\textwidth}
  \begin{algorithmic}[1]
  \label{alg:accel}
\FOR{$t=1,\dots,T$}
\STATE $\v_t = \beta_{t,1}\widetilde{\v}_{t-1} + \beta'_{t,1}{\z}_{t-1}$
\STATE $\nabla \Phi(\z_t) = \nabla \Phi(\z_{t-1})-\eta_t \nabla L(\v_t)$ 
\STATE $\widetilde{\v}_t=\beta_{t,2}\widetilde{\v}_{t-1}+\beta'_{t,2}\z_t$
\ENDFOR
\STATE \textbf{Output}: $\widetilde{\v}_{T}$
\vspace{0.05cm}
  \end{algorithmic}
  \end{minipage}
}
\hfill
\fbox{
  \begin{minipage}[t]{0.53\textwidth}
  \begin{algorithmic}[1]
\FOR{$t=1,\dots,T$}
\STATE $\g_t = \beta_{t,3}\g_{t-1} + \beta'_{t,3}\nabla L(\beta_{t,4}\v_{t-1}+\beta'_{t,4}\s_{t-1})$
\STATE $\s_t = \argmin_{\|\s\|\leq 1}\s^{\top} \g_t$
\STATE ${\v}_t={\v}_{t-1}+\eta_t\s_t$
\ENDFOR
\STATE \textbf{Output}: ${\v}_{T}$

  \end{algorithmic}
  \end{minipage}
}

\hfill

\fbox{
  \begin{minipage}{0.969\textwidth}
  \begin{algorithmic}
    \STATE $\p$-player: 
$\p_t=\argmin_{\p\in\Delta^n}
  \sum_{i=1}^{t-1}\alpha_i\ell_i(\p)+ \alpha_t \ell_{t-1}(\p) + \frac{1}{c}D_{E}(\p,\p_0)$
  \STATE $\w$-player: $\w_t=\argmin_{\w\in\R^d}\sum_{i=1}^{t}\alpha_i h_i(\w)$
  \STATE 
 {\bf Output:} $\widetilde{\w}_T=\sum_{t=1}^T\alpha_t\w_t$
  \end{algorithmic}
  \end{minipage}
}
}
\end{algorithm}
\end{minipage}
\end{figure}
In the preceding subsections, we showed that, with suitable step sizes, steepest descent and average mirror descent can achieve an $\mathcal{O}\left(\frac{\log n\log T}{T}\right)$ margin maximization rate. We now aim to derive even faster rates using two approaches, as illustrated in the top two boxes of Algorithm \ref{alg:accel}. The left box introduces a Nesterov-acceleration-based mirror descent \citep{nesterov1988approach,tseng2008accelerated}: In each iteration $t$, the algorithm initially performs an extra update to yield $\v_t$ (Step 2), then executes a mirror descent step with the gradient at $\v_t$ (Step 3), and finally calculates a moving average (Step 4).
On the other hand, the right box depicts a momentum-based steepest descent algorithm: In each iteration, the method maintains a momentum term $\g_t$ with an additional gradient (Step 2), then identifies the steepest direction with respect to $\g_t$ (Step 3), and applies this direction to update the decision (Step 4). At first glance, these two algorithms appear markedly different. However, we show that with appropriately chosen parameters, they are actually equivalent, in the sense that they both correspond to the online dynamic in the bottom box of Algorithm \ref{alg:accel}. More specifically, we provide the following theoretical guarantee, which is proved in Section \ref{sec:thm:accel}.
\begin{theorem}
\label{thm:accel}
Suppose Assumption \ref{ass:only} holds wrt a general norm $\|\cdot\|$, and $\frac{1}{2}\|\cdot\|^2$ is $\lambda$-strongly convex wrt $\|\cdot\|$. For the left box, let $\beta_{t,1}=\frac{\lambda}{4}$, $\beta'_{t,1}=\frac{\lambda}{2(t-1)}$, 
$\beta_{t,2}=1$, $\beta'_{t,2}=\frac{2}{t+1}$, and  $\eta_t=\frac{t}{L(\v_t)}$. For the right box, let $\beta_{t,3}=\frac{t-1}{t+1}$, $\beta_{t,4}=\frac{\lambda}{4}$, $\beta'_{t,4}=\frac{\lambda t\|\g_{t-1}\|_*}{4}$,
$\beta'_{t,3}=\frac{2}{(t+1)L(\beta_{t,4}\v_{t-1}+\beta'_{t,4}\s_{t-1})}$, and $\eta_t=t\|\g_t\|_*$. For the bottom box, let $c=\frac{\lambda}{4}$, $\alpha_t=t$. Then all three methods in Algorithm \ref{alg:accel} are identical, in the sense that $\widetilde{\v}_T=\v_T=\widetilde{\w}_T$. Moreover, when 
$T\geq \frac{4\sqrt{2\log n}}{\sqrt{\lambda}\gamma}$, we have $$\frac{\min_{\p\in\Delta^n}\p^{\top}A\widetilde{\w}_T}{\|\widetilde{\w}_T\|}\geq \gamma - \frac{32\log n}{\gamma^2T^2\lambda}.$$
\end{theorem}
Theorem \ref{thm:accel} reveals that the two strategies implemented in Algorithm \ref{alg:accel} yield an optimal $\mathcal{O}\left(\frac{\log n}{[\gamma^2T^2]}\right)$ rate.  It is worth noting that a similar online dynamic to the one detailed in the bottom box of Algorithm \ref{alg:accel} was also considered by \citet[Algorithm 5,][]{wang2022accelerated}. Nonetheless, there are some crucial distinctions: 1) Their work only demonstrated that this dynamic could achieve a positive margin, leaving open questions regarding whether the margin can be maximized (i.e., converge to $\gamma$), and if so, what the margin maximization rate would be; 2) They only presented the online dynamic, without its equivalent optimization form.

\section{Implicit Bias in Adversarial Training}
\label{section: game}
In this section, we broaden the scope of the proposed framework in \eqref{eqn:defn:game} to  explore the implicit bias of first-order methods in adversarial training, and provide several state-of-the-art results on the margin maximization rate. 

\subsection{Basic Setting}
We focus on the empirical risk minimization problem as defined in Equation \eqref{eqn:the ERM function}. Our interest lies in examining the optimization trajectory of the following standard Adversarial Training with $\ell_s$-perturbation ($\ell_s$-AT) \cite{goodfellow2014explaining,madry2018towards} (where $s\geq 1$ is a constant chosen by the optimizer). For all training rounds $t\in[T]$, the algorithm is run as below:
\begin{equation}
\begin{cases}
\label{eqn:general:AT}
\textstyle
  &    \hspace{-5pt}\forall i\in[n], \dl^{(i)}_t\leftarrow\argmax\limits_{\|\dl\|_s\leq\epsilon} \left[\exp(-y^{(i)}(\x^{(i)}+\dl)^{\top}\v_{t-1})\right];\\
      &   \hspace{-5pt}  \forall i\in[n], \widetilde{\x}_t^{(i)}\leftarrow\x^{(i)}+\dl^{(i)}_t;\\
      &   \hspace{-5pt} 
 \widetilde{\S}_t\leftarrow \{(\widetilde{\x}_t^{(i)},y^{(i)})\}_{i=1}^n;\\
      &   \hspace{-5pt} 
 \v_t=\mathcal{A}\left(\v_{t-1},\nabla L(\v_{t-1};\widetilde{\S}_{t})\right);
    \end{cases}        
\end{equation}
To be more specific, in each round $t$ of this procedure,  the adversary first generates a noise $\dl^{(i)}_t$ for each data point $i\in[n]$ by maximizing the loss corresponding to the learner's prediction $\v_{t-1}$, and adds it to the corresponding feature vector. The added noise $\dl^{(i)}_t$ is ``small" in the sense that $\|\dl^{(i)}_t\|_s \leq \epsilon$. This perturbed data forms a new data set $\widetilde{\mathcal{S}}_t$. After that, the learner updates the decision by performing some first-order method $\mathcal{A}$ on $\widetilde{\mathcal{S}}_t$. Given the dataset $\S=\{(\mathbf{x}^{(i)},y^{(i)})\}_{i=1}^n$, 
and some noise tolerance $\epsilon > 0$, the \textit{$(2,s)$-mix-norm margin}, or \emph{robust margin} of $\S$ \cite{charles2019convergence,li2020implicit} is defined as 
$$\textstyle \gamma_{2,s} := \min\limits_{\p\in\Delta^n}\min\limits_{\|\dl^{(i)}\|_s\leq \epsilon, \forall i\in[n]} \frac{\sum_{i=1}^np_iy^{(i)}(\x^{(i)}+\dl^{(i)})^{\top}\w}{\|\w\|_2}.$$
We say that a dataset is \textit{linearly separable with $(2,s)$-mix-norm margin $\gamma_{2,s}$} 
when the associated mix-norm margin $\gamma_{2,s}$ is strictly positive. We will use $\w_{2,s}^*$ to refer to any maximizer of the form
$$\textstyle \argmax\limits_{\|\w\|_2\leq 1}\min\limits_{\p\in\Delta^n}\min\limits_{\|\dl^{(i)}\|_s\leq \epsilon, \forall i\in[n]} \sum_{i=1}^np_iy^{(i)}(\x^{(i)}+\dl^{(i)})^{\top}\w.$$
Note that the $(2,s)$-mix-norm max-margin classifier provides the maximal $\ell_2$-normalized margin when perturbed with $\ell_s$-norm bounded noise, which is a natural robust classifier against $\ell_s$-perturbation. Moreover, previous work has shown that GD in $\ell_s$-AT will asymptotically converge to this particular classifier~\cite{li2020implicit}. 
Finally, we have the following relationship between the mix-norm margin and the original margin. 
\begin{lemma}[Condition for $\gamma_{2,s}>0$]
\label{lemmmmm:1}
Let $\gamma_2$ be the $\ell_2$-margin of $\w^*_{\|\cdot\|_2}$. 
If $\gamma_2>0$ and $\epsilon< \frac{\gamma_2}{\|\w_{\|\cdot\|_2}^*\|_r}$, then $\gamma_{2,s}>0$, where $\|\cdot\|_r$ is the dual norm of the $\|\cdot\|_s$-norm.
\end{lemma}
Lemma \ref{lemmmmm:1} can be proven by simply noting that $\w_{\|\cdot\|_2}^*$ achieves a positive mix-norm margin in this case, and thus the mix-norm margin of $\w_{2,s}^*$ should be even larger. 
\subsection{Understanding $\ell_s$-AT Via The Game Framework}

\begin{figure}
\centering
\begin{minipage}{0.9\textwidth}
\begin{protocoll}[H]
\scalebox{0.95}{
\begin{minipage}{0.99\textwidth}
\caption{No-regret dynamics with weighted OCO for solving $g'(\p,\w,\{\dl^{(i)}\}_{i=1}^n)$}

\begin{algorithmic}[1]
\label{pro:no-regret for game:2}
\STATE \textbf{Initialization}: $\textsf{OL}^{\w}$, $\textsf{OL}^{\p}$, $\{\textsf{OL}^{\dl^{(i)}}\}_{i=1}^n$.
\FOR{$t=1,\dots,T$}
\STATE $\w_t\leftarrow \textsf{OL}^{\w}$
\STATE $\textsf{OL}^{\p}\leftarrow\alpha_t,\ell_t(\cdot)$\qquad \qquad // $  \ell_t(\p) = \sum_{i=1}^np_{i}y^{(i)}\x^{(i)\top}\w_t$\\

\STATE 
$\forall i\in[n]$,  $\textsf{OL}^{\dl^{(i)}}\leftarrow\alpha_t,s^{(i)}_t(\cdot)$ \quad
// $ s^{(i)}_t(\dl^{(i)}) = y^{(i)}\dl^{(i)\top}\w_t, \forall i\in[n] $

\STATE $\p_t\leftarrow \textsf{OL}^{\p}$, 
$\forall i\in[n]$,   $\dl_{t}^{(i)}\leftarrow\textsf{OL}^{\dl^{(i)}}$
\STATE $\textsf{OL}^{\w}\leftarrow\alpha_t,h_t(\cdot)$\qquad \quad  // $h_t(\cdot)=-g'\left(\w,\p_t,\{\dl_t^{(i)}\}_{i=1}^n\right)$ 
\ENDFOR
\STATE \textbf{Output}: $\widetilde{\w}_T\leftarrow\sum_{t=1}^{T}\alpha_t\w_t~.$
\end{algorithmic}
\end{minipage}
}
\end{protocoll}
\end{minipage}
\end{figure}
\noindent To accommodate the extra perturbation process, we extend the two-player game in \eqref{eqn:defn:game} to a multi-player game, given by 
\begin{equation}
\label{eqn:the general game}
\max\limits_{\w\in\R^d}\min\limits_{\p\in\Delta^n}\min\limits_{\substack{\|\mathbf{\dl}^{(i)}\|_s\leq \epsilon, \\ \forall i\in[n]}}g'(\p,\w,\{\dl^{(i)}\}_{i=1}^n)=\sum_{i=1}^np_{i}y^{(i)}\x^{(i)\top}\w + \frac{1}{n}\sum_{i=1}^ny^{(i)}\dl^{(i)\top}\w -\frac{1}{2}\|\w\|_2^2~,
\end{equation}
and the corresponding online dynamic is presented in Protocol \ref{pro:no-regret for game:2}. Compared with Protocol  \ref{pro:no-regret for game}, the main difference is that there are $n$ extra $\dl^{(i)}$-players (Steps 5-7) that pick the adversarial noise. To be more specific, after the $\w$-player makes the decision, the $\p$-player and the $\dl^{(i)}$-players observe their loss functions $\ell_t(\p)$ and $s_t^{(i)}(\dl^{(i)})$, 
along with the weight $\alpha_t>0$. After that, these players pick their decision $\p_t$ and $\dl^{(i)}_t$ based on the corresponding online algorithms 
$\textsf{OL}^{\p}$ and  $\textsf{OL}^{\dl^{(i)}}$ they apply. Finally, the $\w$-player observes the weight $\alpha_t$ and its loss $h_t(\cdot)$. Applying Protocol \ref{pro:no-regret for game:2} to solve  \eqref{eqn:the general game} yields the following conclusion analog to Theorem \ref{thm:margin}. The proof is in Section \ref{proof:Theorem:77}.
\begin{theorem}
\label{thm:main:normalized margin}
Suppose Assumption \ref{ass:main:ass} holds with respect to the $\ell_2$-norm, and $\S$ is linearly separable with $(2,s)$-mix-norm margin $\gamma_{2,s}$. Then, applying  Protocol \ref{pro:no-regret for game:2} to solve game \eqref{eqn:the general game} with noise level $\epsilon\in\left[0,\frac{\gamma_2}{\|\w_{\|\cdot\|_2}^*\|_r}\right)$ ensures 
\begin{align}
\min\limits_{\p\in\Delta^n}\min\limits_{\|\dl^{(i)}\|_s\leq \epsilon, \forall i\in[n]} \sum_{i=1}^np_iy^{(i)}(\x^{(i)}+\dl^{(i)})^{\top}\frac{\widetilde{\w}_T}{\|\widetilde{\w}_T\|_2}
\geq  \gamma_{2,s} - \min\left\{\frac{4 C_T}{\gamma^2_{2,s}},\frac{\sum_{t=1}^T\alpha_t}{\|\widetilde{\w}_T\|_2}C_T\right\},
\label{eqn:thm:1:main equality}
\end{align}
and 
\begin{equation}
\label{eqn:theorem 111121}
    \left\|\frac{\widetilde{\w}_T}{\|\widetilde{\w}_T\|}-\w_{2,s}^* \right\|_2\leq   \min\left(\frac{\sum_{t}\alpha_t}{\|\widetilde{\w}_T\|_2},\frac{4}{\gamma_{2,s}^2}\right){\sqrt{8C_T}},
\end{equation}
where $C_T:=\frac{\emph{Reg}_T^{\w}+\emph{Reg}_T^{\p}+\frac{1}{n}\sum_{i=1}^n\emph{Reg}^{\dl^{(i)}}_T}{\sum_{t=1}^T\alpha_t}.$
\end{theorem}

\subsection{$\ell_s$-AT with Gradient Descent}
We begin with the \( \ell_s \)-AT implemented using gradient descent. The specifics of this approach are outlined in the top box of Algorithm \ref{alg:1:GBAT}. This method exemplifies \eqref{eqn:general:AT}, where \( \mathcal{A} \) is set up with gradient descent. Note that this algorithm reduces to the standard FGSM algorithm \citep{goodfellow2014explaining} when $s = \infty.$ For this algorithm, we draw the following conclusion, which is proved in Section \ref{sec:proof:faster rates}. 

\begin{figure}
\centering
\begin{minipage}{0.9\textwidth}
\begin{algorithm}[H]
\caption{$\ell_s$-AT with gradient descent}
{
\fbox{
  \begin{minipage}[t]{0.965\textwidth}
  \begin{algorithmic}[1]
\STATE \textbf{Initialization: $\widetilde{\S}_0=\S$, $\v_0=\mathbf{0}$}
\FOR{$t=1,\dots,T$}
\STATE $\v_t \leftarrow \v_{t-1}-\eta_{t-1}\nabla L\left(\v_{t-1};\widetilde{\S}_{t-1}\right)$
\STATE $\forall i\in[n]$: $\dl^{(i)}_t \leftarrow \argmax\limits_{\|\dl\|_s\leq\epsilon} \exp\left(-y^{(i)}\left(\x^{\left(i\right)}+\dl\right)^{\top}{\v}_t\right)$, $\widetilde{\x}_t^{(i)}\leftarrow\x^{(i)}+\dl_t^{(i)}$
\STATE $\widetilde{\S}_t \leftarrow \left\{\left(\widetilde{\x}_t^{(i)},y^{(i)}\right)\right\}_{i=1}^n$
\ENDFOR
\STATE \textbf{Output}: $\v_{T}$
  \end{algorithmic}
  \end{minipage}
}\\
}
{\fbox{%
     \begin{minipage}[t]{0.975\textwidth}
  \begin{algorithmic}
    \STATE $\w$-player: 
$\w_t =\w_{t-1} - c_{t-1}\alpha_{t-1}\nabla h_{t-1}(\w_{t-1})$
\STATE $\p$-player: 
$ \p_t=   \argmin\limits_{\p\in\Delta^n}\sum_{i=1}^{t}\alpha_i\ell_i(\p)+D_{E}\left(\p,\frac{\mathbf{1}}{n}\right)$
\STATE $\forall i\in[n]$, $\dl^{(i)}$-player: 
$\dl^{(i)}_t=\argmin\limits_{\|\dl\|_p\leq \epsilon}{{\sum_{j=1}^{t}\alpha_js_j^{(i)}(\dl)}}
$ 
\STATE {\bf Output:} $\widetilde{\w}_T=\sum_{t=1}^T\alpha_t{\w}_t.$
  \end{algorithmic}
  \end{minipage}
}

\label{alg:1:GBAT}
}
\end{algorithm}
\end{minipage}
\end{figure}

\begin{theorem}
\label{thm:GDBAT}
Suppose Assumption \ref{ass:main:ass} holds with respect to the $\ell_2$-norm, and $\S$ is linearly separable with $(2,s)$-mix-norm margin $\gamma_{2,s}$. 
{ Let $s\in(1,2]$}. For the top box of Algorithm \ref{alg:1:GBAT}, let the step size $\eta_{t-1}$ be $\eta_{t-1}=\frac{\alpha_{t-1}}{L(\v_{t-1};\widetilde{\S}_{t-1})}$. For the bottom box, let $c_{t-1}=\frac{1}{\alpha_{t-1}}$. Then, the two methods presented in Algorithm \ref{alg:1:GBAT} are equivalent, in the sense that $\widetilde{\w}_T=\v_T$, and $\w_t=\frac{\nabla L(\v_{t-1};\widetilde{\S}_{t-1})}{L(\v_{t-1};\widetilde{\S}_{t-1})}$. Let $\alpha_t=\frac{1}{2}$\ \ for all $t\in[T]$, and  $\epsilon\in\left[0,\frac{\gamma_2}{2\|\w_{\|\cdot\|_2}^*\|_r}\right)$.  Then for the online dynamic, we have $\|\widetilde{\w}_T\|_2\geq T\gamma_2/2$,  and the $(2,s)$-mix-norm margin converge to $\gamma_{2,s}$ on the rate of\ $\mathcal{O}\Biggl( { \frac{\log n + 2(1+\epsilon)^2+\frac{\pi \epsilon^2(d^{\frac{1}{s}-\frac{1}{2}}+\epsilon)^2}{6(s-1)^2\gamma_2^2}}{T\gamma_2/2}}\Biggl).$
\end{theorem}
The first section of Theorem \ref{thm:GDBAT} elucidates the equivalent representation of \( \ell_p \)-AT with gradient descent within the game framework. The specifics of the online dynamic are detailed in the bottom box of Algorithm \ref{alg:1:GBAT}. In this dynamic, the \( \mathbf{w} \)-player acts first, employing the standard online gradient descent algorithm with a step size \( c_t \). Subsequently, the \( \mathbf{p} \)-player and the \( \dl^{(i)} \)-players utilize, respectively, the FTRL\(^+ \) algorithm with a constant parameter, {and the FTL\(^+ \) algorithm}. 

Compared with the results in Section \ref{sec:main result}, the $\dl^{(i)}$-players introduce extra technical challenges. When analyzing the update of the $\dl^{(i)}$-player, one of our key observation is that, to make the algorithm equivalence work, each $\dl^{(i)}$ would need to perform FTL$^+$. However, to make the regret bound small, the $\dl^{(i)}$ would need to use FTRL$^+$, {in order to cancel some extra terms in the $\w$-player's regret caused by the $\dl^{(i)}$-player. Unfortunately, FTL$^+$ and FTRL$^+$ are in general different. We address this issue by providing a novel tighter bound for the $\w$-player by using the property of strongly convex sets.}  
We refer to the proof for more details. 

The final section of Theorem \ref{thm:GDBAT} highlights the margin maximization rates as well as the direction convergence rates associated with the output of Algorithm \ref{alg:1:GBAT}. {
    When \( s  = 2 \), the \( d \)-dependence vanishes, while as $s$ approaches 1, we obtain a linear dependence on $d$. Comparing to the $\mathcal{O}\left(\frac{\log n + \sqrt{d}}{\log T}\right)$ rate provided in \cite{li2020implicit}, the dependence on $T$ is greatly improved.
    }

\vspace{0.2in}

\subsection{$\ell_s$-AT With Nesterov-style Acceleration}

\begin{figure}
\centering
\begin{minipage}{0.9\textwidth}
\begin{algorithm}[H]
\caption{$\ell_s$-AT with Nesterov-style acceleration}
\fbox{
  \begin{minipage}[t]{0.962\textwidth}
  \begin{algorithmic}[1]
\STATE \textbf{Initialization: $\widetilde{\S}_0=\S$, $\v_0=\mathbf{0}$, $\z_0=\textbf{0}$}
\FOR{$t=1,\dots,T$}
\STATE $\widehat{\v}_t = \beta_{t,1}\v_{t-1}+\beta_{t,2}\z_{t-1}$
\STATE $\forall i\in[n]$: $\dl^{(i)}_t \leftarrow \argmax\limits_{\|\dl\|_s\leq\epsilon} \exp\left(-y^{(i)}\left(\x^{\left(i\right)}+\dl\right)^{\top}{\v}_t\right)$, $\widetilde{\x}_t^{(i)}\leftarrow\x^{(i)}+\dl_t^{(i)}$
\STATE $\widetilde{\S}_t \leftarrow \left\{\left(\widetilde{\x}_t^{(i)},y^{(i)}\right)\right\}_{i=1}^n$
\STATE $\z_t = \z_{t-1}-\eta_{t-1}\nabla L\left(\widehat{\v}_t;\widetilde{\S}_t\right)$
\STATE $\v_t=\beta_{t,3}\v_{t-1}+\beta_{t,4}\z_t$
\ENDFOR
\STATE \textbf{Output}: $\v_{T}$
  \vspace{0.1cm}
  \end{algorithmic}
  \end{minipage}
}
\fbox{%
     \begin{minipage}[t]{0.97\textwidth}
  \begin{algorithmic}
\STATE $\p$-player: 
$ \p_t=   \argmin\limits_{\p\in\Delta^n}\sum_{i=1}^{t-1}\alpha_i\ell_i(\p) +\alpha_t\ell_{t-1}(\p)+D_{E}(\p,\frac{\mathbf{1}}{n})$
\STATE $\forall i\in[n]$, $\dl^{(i)}$-player : 
$\dl^{(i)}_t=\argmin\limits_{\|\dl\|_p\leq \epsilon}{ \sum_{j=1}^{t-1}\alpha_js_j^{(i)}(\dl)+\alpha_ts_{t-1}^{(i)}(\dl)}$ 
 \STATE $\w$-player: 
$\w_t = \argmin\limits_{\w\in\R^d} \sum_{i=1}^t\alpha_ih_i(\w)$
\STATE {\bf Output:} $\widetilde{\w}_T=\sum_{t=1}^T\alpha_t{\w}_t.$
  \end{algorithmic}
  \end{minipage}
}

\label{alg:3:FR}
\end{algorithm}
\end{minipage}
\end{figure}

Finally, we show how Nesterov-style acceleration can help obtain faster convergence rates to the mix-norm maximal margin classifier in adversarial training. The method is summarized in Algorithm \ref{alg:3:FR}, for which the following result holds. The proof is provided in Section \ref{sec:proof:faster rates}. 
\begin{theorem}
\label{thm:acccccc}
Suppose Assumption \ref{ass:main:ass} holds with respect to the $\ell_2$-norm, and $\S$ is linearly separable with $(2,s)$-mix-norm margin $\gamma_{2,s}$. For the top box of Algorithm \ref{alg:3:FR}, set $\beta_{t,1}=1,\beta_{t,2}=\frac{2}{t-1}$, $\eta_{t-1}=\frac{t}{2L(\widehat{\v}_t;\S_t)}$, $\beta_{t,3}=1$, and $\beta_{t,4}=\frac{2}{t+1}$. For the bottom box, let $\alpha_t=\frac{t}{2}$. Then, the two methods presented in Algorithm \ref{alg:3:FR} are equivalent, in the sense that $\widetilde{\w}_T=\v_T$. Moreover, if we assume $\epsilon\in\left[0,\frac{\gamma_2}{3\|\w_{\|\cdot\|_{2}}^*\|_r}\right)$, then, { for $s\in(1,2]$}, $t\geq 2$, the normalized robust margin satisfies
\begin{align*}
\textstyle
&\frac{\min\limits_{\p\in\Delta^n}\min\limits_{\|\dl^{(i)}\|_s\leq \epsilon, \forall i\in[n]} \sum_{i=1}^np_iy^{(i)}(\x^{(i)}+\dl^{(i)})^{\top}\widetilde{\w}_T}{\|\widetilde{\w}_T\|_2}
\geq \gamma_{2,s} - \frac{{20+\log n+ \frac{\pi\epsilon^2(d^{\frac{1}{s}-\frac{1}{2}}+\epsilon)^2}{(s-1)^2\gamma_2^2}}}{T^2\gamma_2/3}.
\end{align*}
{For $s\in(2,\infty),$ the convergence rate to 
$\gamma_{2,s}$ is $\textstyle   \mathcal{O}\left(\frac{\log n + (d^{\frac{1}{2}-\frac{1}{s}}\epsilon)^2\epsilon^2}{T}\right).$}
%
\end{theorem}

\section{Proofs of algorithmic equivalences and rates}
\label{section:detailed proof}
\vspace{-0.05in}

In this section, we provide the proofs of the main results in Sections~\ref{sec:main result} and~\ref{section: game}.

\subsection{Proof for Theorems \ref{thm:mirror margin} and \ref{cor:u2222}}
\label{appendix:Proof of section 3}
Here, we present a more general algorithm framework (given in Algorithm \ref{alg:mirror descent:general}) which allows different step sizes. In the following, we first state a general theorem for this algorithm, and 
Theorems \ref{thm:mirror margin} and \ref{cor:u2222} can be obtained by setting $\alpha_t =1$ and $t$ respectively.

\begin{figure}
\centering
\begin{minipage}{0.9\textwidth}
\begin{algorithm}[H]
\scalebox{0.90}{
\caption{Mirror Descent (General Version)}
\fbox{
  \begin{minipage}[t]{0.46\textwidth}
  \begin{algorithmic}[1]
  
\vspace{0.08cm}

\FOR{$t=1,\dots,T$}

\vspace{0.12cm}

\STATE \hspace{-0.2cm}$\nabla \Phi(\v_t)=\nabla\Phi(\v_{t-1}) - \eta_t \nabla L(\v_{t-1})$

\vspace{0.05cm}

\ENDFOR

\vspace{0.1cm}

\STATE \textbf{Output}: $\widetilde{\v}_{T}=\sum_{t=1}^T\frac{\alpha_t}{\sum_{i=1}^t\alpha_i}\v_t$  

  \end{algorithmic}
  \end{minipage}
}
\hfill
\fbox{
  \begin{minipage}[t]{0.55\textwidth}
  \begin{algorithmic}
  \STATE \hspace{-0.25cm}$\p$-player: $\p_t=  \argmin\limits_{\p\in\Delta^n} \alpha_t\ell_{t-1}(\p)+\beta_tD_{E}\left(\p,\frac{\mathbf{1}}{n}\right)$
  \STATE \hspace{-0.25cm}$\w$-player: $\w_t =  \argmin\limits_{\w\in\R^d} \sum_{j=1}^{t}\alpha_j h_j(\w) $
  \vspace{0.1cm}
  \STATE \hspace{-0.25cm}{\bf Output:} $\widetilde{\w}_T=\sum_{t=1}^T\alpha_t\w_t$
  \vspace{0.08cm}
  \end{algorithmic}
  \end{minipage}
}
\label{alg:mirror descent:general}
}
\end{algorithm}
\end{minipage}
\end{figure}

\begin{theorem}
\label{thm:mirror margin:general}
Suppose Assumption \ref{ass:only} holds wrt $\|\cdot\|_q$-norm for $q\in(1,2]$. For the left box of Algorithm \ref{alg:mirror descent:general}, let $\eta_t=\frac{\alpha_{t}}{L(\v_{t-1})}$. For the right box, let $\beta_t$ be $\frac{\alpha_t}{\sum_{i=1}^{t-1}\alpha_i}$ for $t>1$, $\beta_1=\alpha_1$. Then the methods in the two boxes of Algorithm \ref{alg:mirror descent:general} are identical, in the sense that $\widetilde{\v}_T=\widetilde{\w}_T$, and $\v_t= \w_t\cdot\sum_{i=1}^t\alpha_i$.  $\widetilde{\v}_T$ achieves a positive margin (no smaller than $\gamma^2/4$) for sufficiently large $T$ such that 
\begin{equation}
\label{eqn:theorem 2: condition}
\frac{\gamma^2}{4}\sum_{t=1}^T\alpha_t\geq \left( 2\sum_{t=2}^T \frac{\alpha_t^2}{\sum_{j=1}^{t-1}\alpha_j(q-1)} + 2\log n \sum_{t=2}^T\frac{\alpha_t}{\sum_{i=1}^{t-1}\alpha_i}\right)+\alpha_1(1+2\log n).
\end{equation}
After \eqref{eqn:theorem 2: condition} is satisfied, the margin of $\widetilde{\v}_T$ is lower bounded by 
\begin{equation*}
\frac{\min\limits_{\p\in\Delta^n}\p^{\top}A\widetilde{\v}_T}{\left\|\widetilde{\v}_T\right\|_q} \geq \gamma - \frac{4\left[\left( 2\sum_{t=2}^T \frac{\alpha_t^2}{\sum_{j=1}^{t-1}\alpha_j(q-1)} + 2\log n \sum_{t=2}^T\frac{\alpha_t}{\sum_{i=1}^{t-1}\alpha_i}\right)+\alpha_1(1+2\log n)\right]}{\gamma^2\sum_{t=1}^T\alpha_t},
\end{equation*}
and the directional error $\left\| \frac{\widetilde{\w}_T}{\|\widetilde{\w}_T\|_q}-\w^*_{\|\cdot\|}\right\|_q$ is upper bounded by 
\begin{equation*}
    \begin{split} \frac{8}{\gamma^2\sqrt{q-1}}\sqrt{\frac{2\left[\left( 2\sum_{t=2}^T \frac{\alpha_t^2}{\sum_{j=1}^{t-1}\alpha_j(q-1)} + 2\log n \sum_{t=2}^T\frac{\alpha_t}{\sum_{i=1}^{t-1}\alpha_i}\right)+\alpha_1(1+2\log n)\right]}{(q-1)\sum_{t=1}^T\alpha_t}}.        
    \end{split}
\end{equation*}
\end{theorem}
\begin{proof}
We first focus on the algorithm equivalence, and  start from the online learning framework. For $\w$-player, we have 
\[
\w_t 
=  \argmin\limits_{\w\in\R^d} \sum_{j=1}^t -\alpha_j\p_j^{\top}\A\w + \frac{\sum_{j=1}^t\alpha_j}{2}\|\w\|_q^2
= \left[\nabla\Phi\right]^{-1}\Biggl(\frac{1}{\sum_{j=1}^t\alpha_j}\sum_{j=1}^t\alpha_j\A^{\top}\p_j \Biggl),
\]
which implies that 
\begin{equation*}
\nabla \Phi(\w_t) = \frac{1}{\sum_{j=1}^t\alpha_j}\sum_{j=1}^t\alpha_j\A^{\top}\p_j = \frac{\sum_{j=1}^{t-1}\alpha_j}{\sum_{j=1}^t\alpha_j}\nabla \Phi(\w_{t-1})+\frac{\alpha_t}{\sum_{j=1}^t\alpha_j}\A^{\top}\p_t.
\end{equation*}
To proceed, note that 
$\left[\nabla \Phi(\w)\right]_i = \frac{
\text{sign}(w_i)|w_i|^{q-1}}{\|\w\|^{q-2}}.$
Thus, $\forall c>0$, $c\nabla \Phi(\w)=\nabla \Phi(c\w)$ so that, when $c=\sum_{j=1}^t\alpha_j$, we have
\( 
\nabla \Phi\Bigl(\w_t\sum_{j=1}^t\alpha_j\Bigl)=\nabla \Phi\Bigl(\w_{t-1}\sum_{j=1}^{t-1}\alpha_j\Bigl)+\alpha_t\A^{\top}\p_t~.
\)
On the other hand, for the $\p$-player, we have 
$$\p_t=  \argmin\limits_{\p\in\Delta^n}\alpha_t\ell_{t-1}(\p)+\beta_tD_{E}\left(\p,\frac{\mathbf{1}}{n}\right)=\argmin\limits_{\p\in\Delta^n}\frac{\alpha_t}{\beta_t}\p^{\top}\A\w_{t-1}+\sum_{i=1}^n{p_i}\log \frac{p_i}{\frac{1}{n}}.$$
Based on a standard argument on the relationship between OMD with the negative entropy regularizer on the simplex \citep[see, e.g., Section 6.6 of][]{orabona2019modern}, it is easy to verify that $\forall i\in[n], t\in[T]$,
$p_{t,i}=\frac{\exp(-\frac{\alpha_t}{\beta_t}y^{(i)}\x^{(i)\top}\w_{t-1})}{\sum_{j=1}^n\exp(-\frac{\alpha_t}{\beta_t}y^{(j)}\x^{(j)\top}\w_{t-1})},$
where $p_{t,i}$ is the $i$-th element of $\p_t$. Moreover, based on the definition of $L(\w)$, for any $\w\in\R^d$,
$$\frac{\nabla L(\w)}{L(\w)}=-A^{\top}\left[\dots,\frac{\exp(-y^{(i)}\x^{(i)\top}\w)}{\sum_{j=1}^n\exp(-y^{(j)}\x^{(j)\top}\w)},\dots\right]^{\top},$$
which implies that 
$\A^{\top}\p_t = - \frac{\nabla L\left(\frac{\alpha_t}{\beta_t}\w_{t-1}\right)}{L\left(\frac{\alpha_t}{\beta_t}\w_{t-1}\right)}.$ 

Combining the above equations and the definition of $\beta_t=\frac{\alpha_t}{\sum_{i=1}^{t-1}\alpha_i}$, we get 
$$
\nabla \Phi\Bigl(\w_t\sum_{j=1}^t\alpha_j\Bigl)
=
\nabla \Phi\Bigl(\w_{t-1}\sum_{j=1}^{t-1}\alpha_j\Bigl)-\alpha_t\frac{\nabla L\bigl(\w_{t-1}\sum_{j=1}^{t-1}\alpha_j\bigl)}{L\bigl(\w_{t-1}\sum_{j=1}^{t-1}\alpha_j\bigl)}~.
$$
Substituting $\v_t=\w_t\cdot\sum_{j=1}^t\alpha_j$, we get 
$\nabla\Phi(\v_t) =\nabla\Phi(\v_{t-1}) - \alpha_t \frac{\nabla L(\v_{t-1})}{ L(\v_{t-1})},$
and $\widetilde{\w}_T=\sum_{t=1}^T\alpha_t\w_t=\sum_{t=1}^T\frac{\alpha_t}{\sum_{j=1}^t\alpha_j}\v_t$. The proof is finished by replacing $\frac{\alpha_t}{L(\v_{t-1})}$ with $\eta_t$.
Next, we focus on bounding the regret of the two players. For the $\p$-player, let $\p^{*,\ell}=\min_{\p\in\Delta^n}\sum_{t=1}^T\alpha_t\p^{\top}\A\w_t$ be the best decision in hindsight for the online learning problem. We have 
\[
    \begin{split}
    \text{Reg}^{\p}_T 
\overset{(a)}{\leq} {} & \sum_{t=1}^T \left(\alpha_t\p^{*,\ell,\top}\A\w_{t-1} + \beta_t D_{E}\left(\p^{*,\ell},\frac{\mathbf{1}}{n} \right)\right)-\sum_{t=1}^T\alpha_t\p^{*,\ell,\top}\A\w_t \\
{} & + \sum_{t=1}^T\alpha_t\p_t^{\top}\A(\w_t-\w_{t-1})-\sum_{t=1}^T\beta_t D_E\left(\p_t,\frac{\mathbf{1}}{n}\right)\\
\overset{(b)}{\leq} {} & \sum_{t=1}^T \alpha_t \p^{*,\ell,\top}\A(\w_{t-1}-\w_t) + \sum_{t=1}^T\alpha_t \p_t^{\top}\A(\w_t-\w_{t-1}) + 2\log n \sum_{t=1}^T\beta_t\\
\overset{(c)}{\leq} {} & \sum_{t=1}^T \alpha_t \|\p^{*,\ell}\|_1\|\A(\w_{t-1}-\w_t)\|_{\infty} + \sum_{t=1}^T\alpha_t \|\p_t\|_1\|\A(\w_t-\w_{t-1})\|_{\infty} + 2\log n \sum_{t=1}^T\beta_t\\
\overset{(d)}{\leq} {} & 2\sum_{t=1}^T \alpha_t \|\A(\w_{t-1}-\w_t)\|_{\infty} + 2\log n \sum_{t=1}^T\beta_t~\\
\overset{(e)}{\leq} {} &  2 \sum_{t=1}^T \alpha_t \|\w_{t-1}-\w_t\|_q+ 2\log n \sum_{t=1}^T\beta_t\\
\overset{(f)}{\leq} {} & 2\sum_{t=2}^T \frac{2\alpha_t^2}{2\sum_{j=1}^{t-1}\alpha_j(q-1)} + 2\sum_{t=2}^T\frac{\sum_{j=1}^{t-1}\alpha_j(q-1)}{2\cdot2}\|\w_t-\w_{t-1}\|_q^2 + 2\log n \sum_{t=1}^T\beta_t + \alpha_1\|\w_1\|_q~.
\end{split}
\]
In the above, inequality $(a)$ is based on the optimality of $\p_t$, inequality $(b)$ is due to the fact that the negative entropy regularizer is upper bounded, inequality $(c)$ is because of the H\"{o}lder's inequality, inequality $(d)$ is derived from $\p_t,\p\in\Delta^n$, inequality $(e)$ is based on the H\"{o}lder's inequality and $\|\x^{(i)}\|_p\leq 1$ for all $i\in[n]$, and inequality $(f)$ is based on Young's inequality:
$$\sum_{t=2}^T \alpha_t \|\w_{t-1}-\w_t\|_q\leq \sum_{t=2}^T \frac{2\alpha_t^2}{2\sum_{j=1}^{t-1}\alpha_j(q-1)} + \sum_{t=2}^T\frac{\sum_{j=1}^{t-1}\alpha_j(q-1)}{2\cdot2}\|\w_t-\w_{t-1}\|_q^2,$$
where we pick $\frac{\sum_{j=1}^{t-1}\alpha_i(q-1)}{2}$ as the constant of Young's inequality. 

Finally, note that 
$\w_1=\left[\nabla \Phi\right]^{-1}(\A^{\top}\p_1)$, so from  the property of $p$-norm, we have
$$
\alpha_1\|\w_1\|_q = \alpha_1 \|\left[\nabla \Phi\right]^{-1}(\A^{\top}\p_1)\|_q=\alpha_1\|\A^{\top}\p_1\|_p\leq \alpha_1~. 
$$
For the $\w$-player, note that $h_t(\w)$ is $(q-1)$-strongly convex w.r.t. the $\|\cdot\|_q$-norm. We thus have 
$
\textstyle{
\text{Reg}^{\w}_T \leq -\sum_{t=1}^T\frac{(q-1)\sum_{s=1}^{t-1}\alpha_s}{2}\|\w_t-\w_{t-1}\|_q^2~   
}
$.
\end{proof}

\subsection{Proof of Theorem  \ref{thm:mirror momentum}}
\label{section:thm:mirror momentum}
We first focus on the algorithm equivalence. We have 
\[
    \begin{split}
\w_t = \left[\nabla \Phi\right]^{-1}\left(\frac{1}{\sum_{i=1}^t\alpha_i}\left(\sum_{i=1}^{t-1}\alpha_i\A^{\top}\p_i+\alpha_t \A^{\top}\p_{t-1}\right) \right),
    \end{split}
\]
so that\ \ 
\(  
\nabla \Phi(\w_t) = \frac{1}{\sum_{i=1}^t\alpha_i}\left( \left[\sum_{i=1}^{t-1}\alpha_i\right] \nabla\Phi(\w_{t-1}) +\alpha_{t}\A^{\top}\p_{t-1} +  \alpha_{t-1}\A^{\top}(\p_{t-1}-\p_{t-2}) \right).
\)
Therefore, we have
\[
\nabla \Phi\Bigl(\w_t\sum_{i=1}^t\alpha_i\Bigl) =  \nabla \Phi\Bigl(\w_{t-1}\sum_{i=1}^{t-1}\alpha_i\Bigl) +   \alpha_{t}\A^{\top}\p_{t-1} +  \alpha_{t-1}\A^{\top}(\p_{t-1}-\p_{t-2})~.   
\]
On the other hand, for the $\p$-player, 
based on the relationship between OMD with the negative entropy regularizer on the simplex \citep{Intro:Online:Convex}, it is easy to verify that $\forall i\in[n], t\in[T]$,
$p_{t,i}=\frac{\exp(-\frac{\alpha_t}{\beta_t}y^{(i)}\x^{(i)\top}\w_{t})}{\sum_{j=1}^n\exp(-\frac{\alpha_t}{\beta_t}y^{(j)}\x^{(j)\top}\w_{t})},$
which implies that 
$\A^{\top}\p_t = - \frac{\nabla L\left(\frac{\alpha_t}{\beta_t}\w_{t}\right)}{L\left(\frac{\alpha_t}{\beta_t}\w_{t}\right)}.$ Combining the equations above and replace $\sum_{k=1}^t\alpha_k\w_k$ with $\v_t$, we get 
$$
\nabla \Phi(\v_t) =\nabla \Phi(\v_{t-1}) - \alpha_t  \frac{\nabla L({\v}_{t-1})}{L({\v}_{t-1})} - \alpha_{t-1}\left( \frac{\nabla L({\v}_{t-1})}{L({\v}_{t-1})} - \frac{\nabla L({\v}_{t-2})}{L({\v}_{t-2})}\right).
$$
The proof is finished by setting $\alpha_t=t$. 

Next, we focus on the regret. For the $\w$-player, Note that $h_t(\w)$ is $(q-1)$-strongly convex wrt the $\|\cdot\|$-norm. Let 
$\widehat{\w}_t = \argmin\limits_{\w\in\R^d}\sum_{i=1}^{t-1}\alpha_ih_i(\w)$. Then, based on \cite{wang2021no}, for the regret of the $\w$-player, we have 
\begin{equation*}
\begin{split}
\text{Reg}_T^{\w} &\leq  \sum_{t=1}^T\alpha_t \left(h_t(\w_t)-h_t(\widehat{\w}_{t+1}) -h_{t-1}(\w_t) + h_{t-1}(\widehat{\w}_{t+1})  \right)  - \sum_{t=1}^T\frac{\sum_{i=1}^t\alpha_i(q-1)}{2}\|\w_t-\widehat{\w}_{t+1}\|_q^2   \\
& \leq   \sum_{t=1}^T\alpha_t\|(\p_t-\p_{t-1})^{\top}\A\|_p\|\w_t-\widehat{\w}_{t+1}\|_q - \sum_{t=1}^T\frac{\sum_{i=1}^t\alpha_i(q-1)}{2}\|\w_t-\widehat{\w}_{t+1}\|_q^2\\
&\leq  \sum_{t=1}^T\frac{\alpha_t^2}{2\sum_{i=1}^t\alpha_i(q-1)}\|\A^{\top}(\p_t-\p_{t-1})\|_p^2 + \frac{\sum_{i=1}^t\alpha_i(q-1)}{2}\|\w_t-\widehat{\w}_{t+1}\|_q^2 \\
&\qquad -\sum_{t=1}^T\frac{\sum_{i=1}^t\alpha_i(q-1)}{2}\|\w_t-\widehat{\w}_{t+1}\|_q^2
 \leq \frac{1}{2(q-1)}\sum_{t=1}^T\frac{\alpha_t^2}{\sum_{i=1}^t\alpha_i}\left\|\p_t-\p_{t-1}\right\|_1^2~,
\end{split}    
\end{equation*}
where the first inequality is based on H\"{o}lder's inequality, the second inequality is based on Young's inequality.

For the $\p$-player, we have 
\begin{equation*}
    \begin{split}
{} \sum_{t=1}^T\alpha_t\p_t^{T}\A\w_t &-\min\limits_{\p\in\Delta^n}\alpha_t\p^{\top}\A\w_t \\
= {} &    \sum_{t=1}^T(\alpha_t \p^{\top}_t\A\w_t + \beta_t D_{E}\left(\p_t,\frac{\mathbf{1}}{n}\right) - \min\limits_{\p\in\Delta^n}\sum_{t=1}^T\alpha_t\p^{\top}\A\w_t - \sum_{t=1}^T\beta_t D_E\left(\p_t,\frac{\mathbf{1}}{n}\right) \\
\leq {} & 2\sum_{t=1}^T\beta_t\log n=2\sum_{t=1}^T\frac{\alpha_t}{\sum_{i=1}^t\alpha_i}\log n.
    \end{split}
\end{equation*}

Finally, we focus on the margin and implicit bias. Since $\alpha_t=t$, for the $\w$-player's regret, we have 
$\frac{1}{2(q-1)}\sum_{t=1}^T\frac{\alpha_t^2}{\sum_{i=1}^t\alpha_i}\left\|\p_t-\p_{t-1}\right\|_1^2\leq \frac{1}{(q-1)}\sum_{t=1}^T\|\p_t-\p_{t-1}\|^2_1,$
and for the $\p$-player, we have 
$\sum_{t=1}^T\frac{\alpha_t}{\sum_{i=1}^t\alpha_i}\log n \leq 4\log T\log n.  $
Following Theorem \ref{thm:margin}, we have the margin and implicit bounds when 
\begin{align*}
 \sum_{t=1}^T\alpha_t = \frac{T(T+1)}{2}\geq \frac{T^2}{2} \geq {} & \frac{4}{\gamma^2} \left(4\log n\log T + \frac{2T}{(q-1)} \right)\\
 \geq {} & \frac{4}{\gamma^2} \left(4\log n\log T + \frac{1}{(q-1)}\sum_{t=1}^T\|\p_t-\p_{t-1}\|_1^2 \right),
\end{align*}
since the RHS is exactly $\frac{\gamma^2}{4}(\text{Reg}^{\p}_{T}+\text{Reg}^{\w}_{T})$.

\subsection{Proof for Theorem \ref{thm:steep:acc::main}}
\label{appendix:sddddd}
In this section, we provide the proof related to the steepest descent algorithm. We first restate Algorithm \ref{alg:steepet descent:main}, which is presented in Algorithm \ref{alg:steepet descent}. Here, we provide two online dynamics under the game framework. They both are equivalent to the steepest descent algorithm in the left box, in the sense that $\v_T=\widetilde{\w}_T$. The left one is good for recovering the results in \citet{nacson2019convergence}, while the bottom one is more suitable for analyzing our accelerated rates. Before starting the proof, we introduce the following lemma.  

\begin{figure}
\centering
\begin{minipage}{0.9\textwidth}
\begin{algorithm}[H]
\caption{Steepest Descent [Recall $\ell_t(\p)=g(\p,\w_t),$ and $h_t(\w)=-g(\p_t,\w)$]}
\scalebox{0.89}{\fbox{
  \begin{minipage}[t]{0.46\textwidth}
  \begin{algorithmic}[1]

\FOR{$t=1,\dots,T$}
\STATE $\s_{t-1} = \argmin_{\|\s\|\leq 1} \s^{\top}\nabla L(\v_{t-1})$
\STATE $\v_{t}=\v_{t-1} + \eta_{t-1} \s_{t-1}$
\ENDFOR
\STATE \textbf{Output}: $\v_{T}$
  
  \end{algorithmic}
  \end{minipage}
}
\hfill
\fbox{
  \begin{minipage}[t]{0.59\textwidth}
  \begin{algorithmic}
\vspace{0.1cm}
  \STATE $\p$-player: $\p_t=\argmin\limits_{\p\in\Delta^n}\sum_{i=1}^{t-1}\alpha_i\ell_i(\p)+D_{E}\left(\p,\frac{\mathbf{1}}{n}\right)$
  \vspace{0.1cm}
  \STATE $\w$-player: $\w_t=\argmin_{\w\in\R^d}\alpha_th_t(\w)$
  \vspace{0.09cm}
  \STATE {\bf Output:} $\widetilde{\w}_T=\sum_{t=1}^T\alpha_t\w_t$
  \vspace{0.1cm}
  \end{algorithmic}
  \end{minipage}
}
}
\noindent\scalebox{0.992}{\fbox{%
    \parbox{0.99\textwidth}{%
  \begin{algorithmic}
    \STATE $\w$-player: 
$\w_t =  \argmin\limits_{\w\in\R^d} \left\langle \delta_{t-1}\alpha_{t-1}\nabla h_{t-1}(\w_{t-1}),\w\right\rangle + D_{\frac{1}{2}\|\cdot\|^2}(\w,\w_{t-1})$
  \STATE $\p$-player: $ \p_t=   \argmin\limits_{\p\in\Delta^n}\sum_{i=1}^{t}\alpha_i\ell_{i}(\p)+D_{E}\left(\p,\frac{\mathbf{1}}{n}\right)$
  \STATE 
 {\bf Output:} $\widetilde{\w}_T=\sum_{t=1}^T\alpha_t{\w}_t.$
  \end{algorithmic}
    }%
}
}
\label{alg:steepet descent}

\end{algorithm}
\end{minipage}
\end{figure}

\begin{lemma}
\label{lem:steepst and mirror}
Let $\|\cdot\|$ be any norm in $\R^d$. Let $
\a\in\R^d$, and
\begin{equation}
\label{eqn:theorem5 ssss}
    \s=\argmax\limits_{\|\s'\|\leq 1}\s^{'\top}\a.
\end{equation}
Then 
\begin{equation}
\label{eqn:eqn:theorem5 sssss} \|\a\|_*\s=\argmin\limits_{\x\in\R^d}-\a^{\top}\x+\frac{1}{2}\|\x\|^2.   
\end{equation}
\end{lemma}
\begin{proof}
We first focus on \eqref{eqn:eqn:theorem5 sssss}. Note that the objective has two terms. For the first term, based on H\"{o}lder's inequality, we have 
    $-\a^{\top}\x\geq -\|\a\|_*\|\x\|.$
Let $\|\x\|=c$, where $c>0$ is a constant, then the equality is achieved (and thus the first term of the objective function is minimized) when 
\begin{equation}
    \label{eqn:theorme5 eqeqeqeq}
    \x=\argmin_{\|\x'\|\leq c}-\x^{'\top}\a=\argmax_{\|\x'\|\leq  c}\x^{'\top}\a.
\end{equation}
In this case, for the objective function of  \eqref{eqn:eqn:theorem5 sssss}, we have $-\a^{\top}\x+\frac{1}{2}\|\x\|^2=-c\|\a\|_*+\frac{1}{2}c^2$. It's easy to see that the objective function is minimized when $c=\|\a\|_*.$ The proof is finished by combining \eqref{eqn:theorem5 ssss} and \eqref{eqn:theorme5 eqeqeqeq}.
\end{proof}

We first focus on the bottom box. For the $\w$-player, we show that, if we set $\delta_{t-1}=\frac{1}{\alpha_{t-1}}$, then this OMD algorithm is equivalent to the best response algorithm. 

Specifically, note that $\Phi(\w)=\frac{1}{2}\|\w\|^2$ is now $\lambda$-strongly convex. Thus the corresponding mirror map is well-defined and unique, and the function $\nabla \Phi(\cdot)$ is invertible.   Therefore, the solution for best response is 
\begin{equation}
\label{eqn:steep:meee:1}  \widehat{\w}_t=\argmin_{\w}\alpha_{t-1}h_{t-1}(\w)=\argmin_{\w}h_{t-1}(\w)=\nabla \Phi^{-1}(\A^{\top}\p_{t-1}).
\end{equation}
On the other hand, since the $\w$-player uses OMD, and the decision set is unbounded, we have 
\begin{equation}
    \begin{split}
    \label{eqn:steep:meee:2}
\nabla \Phi(\w_t)= \nabla \Phi(\w_{t-1})-\delta_{t-1}\alpha_{t-1}\nabla h_{t-1}(\w_{t-1}) =\A^{\top}\p_{t-1}.     
    \end{split}
\end{equation}
Note that $\delta_{t-1}\alpha_{t-1}=1$. Combining \eqref{eqn:steep:meee:1} and \eqref{eqn:steep:meee:2}, we can draw the conclusion that $\w_t$ and $\widehat{\w}_t$ are identical, which shows OMD (with $\delta_{t-1}=\frac{1}{\alpha_{t-1}}$) and BR (at round $t-1$) here  are the same. We use the BR  form the algorithm equivalence analysis, and OMD form for the regret analysis. Next, we prove the algorithm equivalence of the left and bottom boxes in Algorithm \ref{alg:steepet descent}. Firstly,  For the $\p$-player, based on the connection between FTRL and EWA, we have 
$$p_{t,i}\propto\exp\left( -y^{(i)}\x^{(i)\top}\left( \sum_{j=1}^{t}\alpha_j\w_j \right)\right) = \exp\left( -y^{(i)}\x^{(i)\top}\widetilde{\w}_{t}\right).$$ 
Combining with the definition of $L$, it implies that $\frac{\nabla L(\widetilde{\w}_{t})}{L(\widetilde{\w}_{t})}=-\A^{\top}\p_t,$
That is, $\nabla L(\widetilde{\w}_{t})=-L(\widetilde{\w}_{t})\A^{\top}\p_t$. Let 
\begin{equation}
\label{eqn:proof:defn:ssss}
\widetilde{\s}_t=\argmax_{\|\s\|\leq 1} -\s^{\top}\frac{\nabla L(\widetilde{\w}_{t})}{L(\widetilde{\w}_{t})}=\argmin_{\|\s\|\leq 1} \s^{\top}\nabla L(\widetilde{\w}_{t}),
\end{equation}
Combining the first equality in \eqref{eqn:proof:defn:ssss} and Lemma \ref{lem:steepst and mirror}, we have 
$$\frac{\|\nabla L(\widetilde{\w}_{t})\|_*}{L(\widetilde{\w}_{t})}\widetilde{\s}_t=\argmin_{\w\in\R^d} \w^{\top}\frac{\nabla L(\widetilde{\w}_{t})}{L(\widetilde{\w}_{t})}+\frac{1}{2}\|\w\|^2=\argmin_{\w\in\R^d}-\p_t^{\top}\A\w+\frac{1}{2}\|\w\|^2=\w_{t+1}.$$
Thus, \ 
\(
\w_t = \frac{\|\nabla L(\widetilde{\w}_{t-1})\|_*}{L(\widetilde{\w}_{t-1})}\widetilde{\s}_{t-1}=\frac{\|\nabla L(\widetilde{\w}_{t-1})\|_*}{(L(\widetilde{\w}_{t-1}))}\argmin_{\|\s\|\leq 1} \s^{\top}\nabla L(\widetilde{\w}_{t-1}).
\)
Finally, we have 
$$\widetilde{\w}_t=\widetilde{\w}_{t-1}+\alpha_t\w_t=\widetilde{\w}_{t}+\alpha_t\frac{\|\nabla L(\widetilde{\w}_{t-1})\|_*
}{L(\widetilde{\w}_{t-1})}\argmin_{\|\s\|\leq 1}\s^{\top}\nabla L(\widetilde{\w}_{t-1}).$$
We can finish the first part of the proof by replacing $\widetilde{\w}_t$ with $\v_t$, $\alpha_t\frac{\|\nabla L(\widetilde{\w}_{t-1})\|_*
}{L(\widetilde{\w}_{t-1})}$ with $\eta_{t-1}$, and $\argmin_{\|\s\|\leq 1}\s^{\top}\nabla L(\widetilde{\w}_{t-1})$ with $\s_t$. Next, we focus on regret. For the $\w$-player, it uses the OMD algorithm, and we set the initial point $\w_0=\mathbf{0}$. Note that we fixed $\alpha_t=\frac{\lambda}{4}$ for all $t$, and thus step size $\delta_{t-1}=\frac{1}{\alpha_{t-1}}=\frac{4}{\lambda}$ is also fixed. Therefore,  the regret bound of OMD \cite[Theorem 6.8,][]{orabona2019modern} can be applied. Define
$\u=\argmin\limits_{\w\in\R^d}\sum_{t=1}^T\alpha_th_t(\w)$. We have 
\begin{align*}
\sum_{t=1}^{T} \alpha_th_t(\w_t) -  \sum_{t=1}^T \alpha_t h_t(\u) 
\leq {} & \frac{\Phi(\u)}{\delta} + \sum_{t=1}^T\frac{\delta\alpha_t^2}{\lambda}\left\|\nabla h_{t}(\w_t) \right\|^2_*\\
= {} & \alpha_T\Phi(\u) + \sum_{t=1}^T\frac{\alpha_t}{\lambda}\Biggl\| \sum_{i=1}^ny^{(i)}\x^{(i)}(p_{t,i}-p_{t-1,i}) \Biggl\|_*^2\\
\leq {} & \alpha_T{\Phi(\u)} + \sum_{t=1}^T\frac{\alpha_t}{\lambda}\left\|\p_t-\p_{t-1}\right\|^2_{1},
\end{align*}
where the second-to-last inequality is derived using triangle inequality and the assumption that $\|\x^{(i)}\|_*$ is upper bounded by 1.  
Next, for $\u$, note that 
\[
\argmin\limits_{\w\in\R^d}\sum_{t=1}^T\alpha_th_t(\w)
= \argmin\limits_{\w\in\R^d}-\frac{1}{\sum_{t=1}^T\alpha_t}\sum_{t=1}^T\alpha_t\p_t^{\top}\A\w+\frac{1}{2}\|\w\|^2~. 
\]
Based on Lemma \ref{lem:steepst and mirror}, we have $\u=\left\|\A^{\top}\left(\frac{1}{\sum_{t=1}^T\alpha_t}\sum_{t=1}^T\alpha_t\p_t\right)\right\|_*\s,$
where 
$\s=-\argmax_{\|\s\|\leq 1} \s^{\top}\left(\A^{\top}\left(\frac{1}{\sum_{t=1}^T\alpha_t}\sum_{t=1}^T\alpha_t\p_t\right)\right).$ Therefore, $\Phi(\u)\leq\frac{1}{2}.$ 

Finally, for the $\p$-player, since it uses FTRL$^+$, we have 
$\text{Reg}^{\p}_T\leq \log n - \sum_{t=1}^T\frac{1}{2}\|\p_t-\p_{t-1}\|_1^2.$
To summarize, when $\alpha_t=\frac{\lambda}{2}$, we have obtained
$\frac{\text{Reg}_{T}^{\w}+\text{Reg}_{T}^{\p}}{\sum_{t=1}^T\alpha_t} = \frac{\frac{\lambda}{4}+\log n }{T\lambda}.$

\subsection{Proof of Theorem \ref{thm:accel}}
\label{sec:thm:accel}
We first focus on the equivalence between the left and bottom boxes of Algorithm \ref{alg:accel}. For the $\w$-player, similar to the proof of Theorem \ref{thm:mirror margin:general}, we have 
\[
\w_t =   \argmin\limits_{\w\in\R^d}\sum_{j=1}^{t}\alpha_j h_j(\w) 
= \left[\nabla\Phi\right]^{-1}\Biggl(\frac{1}{\sum_{j=1}^t\alpha_j}\sum_{j=1}^t\alpha_j\A^{\top}\p_j \Biggl),
 \]
which implies that 
$\nabla \Phi(\w_t) = \frac{\sum_{j=1}^{t-1}\alpha_j}{\sum_{j=1}^t\alpha_j}\nabla \Phi(\w_{t-1})+\frac{\alpha_t}{\sum_{j=1}^t\alpha_j}\A^{\top}\p_t,$
and thus
$$
\nabla \Phi\Bigl(\w_t\sum_{j=1}^t\alpha_j\Bigl)=\nabla \Phi\Bigl(\w_{t-1}\sum_{j=1}^{t-1}\alpha_j\Bigl)+\alpha_t\A^{\top}\p_t~.
$$
On the other hand, for the $\p$-player, since it performs Optimistic FTRL on a simplex with the negative entropy regularizer, we have 
\[
 p_{t,i}\propto {}  \exp\Biggl(-c\left(\sum_{i=1}^{t-1} \alpha_i\w_i+\alpha_t\w_{t-1}\right)^{\top}\x^{(i)}y^{(i)}\Biggl),
\]
which implies that 
$ \frac{\nabla L(c\widetilde{\w}_{t-1}+c\alpha_t\w_{t-1})}{L(c\widetilde{\w}_{t-1}+c\alpha_t\w_{t-1})} = -\A^{\top}\p_t.$
Let $\z_t=\w_t\sum_{i=1}^{t}\alpha_i$, then we have 
$$\textstyle \nabla \Phi\left(\z_t\right)=\nabla \Phi\left(\z_{t-1}\right)-\alpha_t \frac{\nabla L(c\widetilde{\w}_{t-1}+c\alpha_t\w_{t-1})}{L(c\widetilde{\w}_{t-1}+c\alpha_t\w_{t-1})}=\nabla \Phi\left(\z_{t-1}\right)-\alpha_t \frac{\nabla L(c\widetilde{\w}_{t-1}+\frac{c\alpha_t}{\sum_{i=1}^{t-1}\alpha_i}\z_{t-1})}{L(c\widetilde{\w}_{t-1}+\frac{c\alpha_t}{\sum_{i=1}^{t-1}\alpha_i}\z_{t-1})}.$$
Finally, notice that $\widetilde{\w}_t=\widetilde{\w}_{t-1}+\alpha_{t}\w_{t}=\widetilde{\w}_{t-1}+\frac{\alpha_t}{\sum_{i=1}^t\alpha_i}\z_{t}.$

The proof is finished by replacing $\widetilde{\w}_t$ with $\widetilde{\v}_t$, $\w_t$ with $\frac{\z_t}{\sum_{i=1}^t\alpha_i}$, $c\widetilde{\w}_{t-1}+c\alpha_t\w_{t-1}$ with $\v_t$, configuring $\beta_{t,1}=c=\frac{\lambda}{4}$, $\beta'_{t,1}=\frac{c\alpha_{t}}{\sum_{i=1}^{t-1}\alpha_i}=\frac{\lambda}{2(t-1)}$, 
$\beta_{2,t}=1$, $\beta'_{2,t}=\frac{\alpha_t}{\sum_{i=1}^t\alpha_t}=\frac{2}{t+1}$, $\eta_t=\frac{t}{L(\v_t)}$.

Next, we focus on the equivalence between the right and bottom boxes. Note that for the $\w$-player, we also have 
\[
\w_t =   \argmin\limits_{\w\in\R^d}\sum_{j=1}^{t}\alpha_j h_j(\w) = \argmin\limits_{\w\in\R^d}  -\frac{1}{\sum_{j=1}^t\alpha_j}\sum_{j=1}^t\alpha_j\p_j^{\top}\A\w + \frac{1}{2}\|\w\|^2.
\]
Let 
$\s_t=-\argmax_{\|\s\|\leq 1}\s^{\top}\left[\frac{1}{\sum_{j=1}^t\alpha_j} \sum_{j=1}^t \alpha_j\A^{\top}\p_j\right].$
Based on Lemma \ref{lem:steepst and mirror}, we have 
$\left\|\frac{1}{\sum_{j=1}^t\alpha_j} \sum_{j=1}^t \alpha_j\A^{\top}\p_j\right\|_*\s_t = \w_t. $
Next, let $\g_t=-\frac{1}{\sum_{j=1}^t\alpha_j}\sum_{j=1}^t\alpha_j\A^{\top}\p_j$, we know
$ \g_t = \frac{\sum_{j=1}^{t-1}\alpha_j}{\sum_{j=1}^{t}\alpha_j}\g_t + \left(- \frac{\alpha_t}{\sum_{j=1}^t\alpha_j}\A^{\top}\p_t\right).$
For the $\p$-player, it is clear that due to the optimistic term, we have 
$-\A^{\top}\p_t =\frac{\nabla L(c\widetilde{\w}_{t-1}+c\alpha_t\w_{t-1})}{L(c\widetilde{\w}_{t-1}+c\alpha_t\w_{t-1})}.$ Hence, when $\alpha_t=t$, we can conclude the proof by the following algorithm:
\begin{equation*}
    \begin{split}
\g_t  {} &=\frac{t-1}{t+1}\g_{t-1} + \frac{2}{t+1}\frac{\nabla L(c\widetilde{\w}_{t-1}+c{t}{\|\g_{t-1}\|_*}\s_{t-1})}{L(c\widetilde{\w}_{t-1}+c{t}{\|\g_{t-1}\|_*}\s_{t-1})}\\
\s_t  {} & =-\argmax_{\|\s\|\leq 1} -\s^{\top}\g_t=\argmin_{\|\s\|\leq 1} \s^{\top}\g_t\\
\widetilde{\w}_t &= {}  \widetilde{\w}_{t-1}  + {t}{\|\g_t\|_*}\s_t,
    \end{split}
\end{equation*}
with $\beta_{t,3}=\frac{t-1}{t+1}$, $\beta_{t,4}=\frac{\lambda}{4}$, $\beta'_{t,4}=\frac{\lambda t\|\g_{t-1}\|_*}{4}$,
$\beta'_{t,3}=\frac{2}{(t+1)L(\beta_{t,4}\v_{t-1}+\beta'_{t,4}\s_{t-1})}$, $\eta_t=t\|\g_t\|_*$. Finally, we focus on the regret bound. For the $\w$-player, note that $\Phi(\w)$ is $\lambda$-strongly convex with respect to $\|\cdot\|$. Thus, based on \citet[Lemma 3,][]{wang2021no}, we have 
\[
\sum_{t=1}^T\alpha_th_t(\w_t) -    \sum_{t=1}^T\alpha_th_t(\w) \leq - \sum_{t=1}^T\frac{\lambda(t-1)}{4}\|\w_t-\w_{t-1}\|^2.   
\]
On the other hand, note that ${c}=\frac{\lambda}{4}$, so based on \citet[Lemma 7.35,][]{orabona2019modern}, we have 
\[ 
\sum_{t=1}^T\alpha_t\ell_t(\p_t) -    \sum_{t=1}^T\alpha_t\ell_t(\p) 
\leq  \frac{4{\log n}}{\lambda} + \frac{\lambda}{8}\sum_{t=1}^Tt^2\|\w_t-\w_{t-1}\|^2.
\]
It is easy to verify that $\frac{t^2}{8}\leq \frac{t(t-1)}{4}$ for $t\geq 2$. So to summarize we get 
$C_T=\frac{8\log n}{\lambda T^2}.$
The proof can be finished by plugging in Theorem \ref{thm:margin}.

\subsection{Proof of Theorem~ \ref{thm:main:normalized margin}}
\label{proof:Theorem:77}
In this section, we provide the proofs of the main results in Section~\ref{section: game}. Firstly, following very similar arguments as in Lemma \ref{lem:margin:gap}, we can obtain the following lemma. 
\begin{lemma}
\label{lem:bound on m(w)}
Let $
m(\w)=\min\limits_{\p\in\Delta^n}\min\limits_{\|\dl^{(i)}\|_s\leq \epsilon, \forall i\in[n]}g'(\p,\w,\{\dl^{(i)}\}_{i=1}^n)~,
$ 
and 
$\overline{\w}_T 
= 
\frac{\widetilde{\w}_T}{\sum_{t=1}^T\alpha_t} 
= 
\frac{\sum_{t=1}^{T}\alpha_t\w_t}{\sum_{t=1}^T\alpha_t}
$ 
be the weighted average of the decisions of the $\w$-player across the $T$ rounds. We have  $m(\w) - m(\overline{\w}_T)
\leq C_T$\, $\forall \w\in\R^d$.    
\end{lemma}
We can write 
\begin{align*}
m(\w)= {} & \min\limits_{\p\in\Delta^n}\min\limits_{\|\dl^{(i)}\|_s\leq \epsilon, \forall i\in[n]} \sum_{i=1}^np_{i}y^{(i)}\x^{(i)\top}\w + \frac{1}{n}\sum_{i=1}^ny^{(i)}\dl^{(i)\top}\w-\frac{1}{2}\|\w\|_2^2\notag \\
= {} &  \min\limits_{\p\in\Delta^n}\sum_{i=1}^np_{i}y^{(i)}\x^{(i)\top}\w + \min\limits_{\|\dl^{(i)}\|_s\leq \epsilon, \forall i\in[n]} \frac{1}{n}\sum_{i=1}^ny^{(i)}\dl^{(i)\top}\w-\frac{1}{2}\|\w\|_2^2 \notag\\
= {} &  \min\limits_{\p\in\Delta^n}\sum_{i=1}^np_{i}y^{(i)}\x^{(i)\top}\w - \epsilon\|\w\|_r-\frac{1}{2}\|\w\|_2^2 \notag\\
= {} & \min\limits_{\p\in\Delta^n}\sum_{i=1}^np_{i}y^{(i)}\x^{(i)\top}\w - \epsilon\left(\sum_{i=1}^np_i\|\w\|_r\right)-\frac{1}{2}\|\w\|_2^2 \notag\\
= {} & \min\limits_{\p\in\Delta^n}\sum_{i=1}^n p_i\left(y^{(i)}\x^{(i)}-\epsilon\|\w\|_r\right) -\frac{1}{2}\|\w\|_2^2 \notag\\
= {} & \min\limits_{\p\in\Delta^n}\sum_{i=1}^n\left(\min\limits_{\|\dl^{(i)}\|_s\leq \epsilon}p_i\left(y^{(i)}(\x^{(i)}+\dl^{(i)})^{\top}\w\right) \right) -\frac{1}{2}\|\w\|_2^2 \notag\\
= {} & \min\limits_{\p\in\Delta^n}\min\limits_{\|\dl^{(i)}\|_s\leq \epsilon, \forall i\in[n]} \sum_{i=1}^np_iy^{(i)}(\x^{(i)}+\dl^{(i)})^{\top}\w -\frac{1}{2}\|\w\|_2^2.
\end{align*}
%
Next, following similar procedure as in \eqref{eqn:lower bound of bar wT}, we can also write $m(\overline{\w}_T) \geq  
\gamma_{2,s}\|\overline{\w}_T\|_2  - \frac{1}{2}\|\overline{\w}_T\|_2^2  - C_T~.$
Let $\widetilde{\w}_T = \sum_{t=1}^T \alpha_t \w_t$, 
Then, it implies that
\[
\frac{\min\limits_{\p\in\Delta^n}\min\limits_{\|\dl^{(i)}\|_s\leq \epsilon, \forall i\in[n]} \sum_{i=1}^np_iy^{(i)}(\x^{(i)}+\dl^{(i)})^{\top}\widetilde{\w}_T}{\sum_{t=1}^T\alpha_t} 
\geq \gamma_{2,s} \left\|\frac{\widetilde{\w}_T}{\sum_{t=1}^T\alpha_t}\right\|_2 -C_T~,
\]
where $- \frac{1}{2}\|\overline{\w}_T\|_2^2$ is canceled on both sides.
Dividing both sides by $\|\widetilde{\w}_T\|_2$ gives us
%
\[
\frac{\min\limits_{\p\in\Delta^n}\min\limits_{\|\dl^{(i)}\|_s\leq \epsilon, \forall i\in[n]} \sum_{i=1}^np_iy^{(i)}(\x^{(i)}+\dl^{(i)})^{\top}\widetilde{\w}_T}{\|\widetilde{\w}_T\|_2} \geq \gamma_{2,s} -\frac{\text{Reg}_T^{\w}+\text{Reg}_T^{\p}+\frac{1}{n}\sum_{i=1}^n\text{Reg}^{\dl^{(i)}}_T}{\|\widetilde{\w}_T\|_2}~. 
\]
%
Thus, it suffices to lower bound the norm of $\widetilde{\w}_T$.
Based on Lemma \ref{lem:bound on m(w)} (now applied to $\w = \gamma_{2,s} \w^*_{2,s}$), we have $m(\overline{\w}_T) \geq  \frac{1}{2}(\gamma_{2,s})^2+C_T~. $  
Therefore,
\begin{align*}
\min\limits_{\p\in\Delta^n}&\min\limits_{\|\dl^{(i)}\|_s\leq \epsilon, \forall i\in[n]} \sum_{i=1}^np_iy^{(i)}(\x^{(i)}+\dl^{(i)})^{\top}\widetilde{\w}_T
\geq \frac{\|\widetilde{\w}_T\|_2^2}{2\sum_{t=1}^T\alpha_t} + \frac{\sum_{t=1}^T\alpha_t}{2}\gamma^2_{2,p} - C_T\sum_{t=1}^T\alpha_t~.
\end{align*}
On the other hand, similar to \eqref{eqn:proof:main:2}, we can get 
%
\[
    \|\widetilde{\w}_T\|_2\geq \|y\x\|_2\|\widetilde{\w}_T\|_2\geq \min\limits_{\p\in\Delta^n}\min\limits_{\|\dl^{(i)}\|_s\leq \epsilon, \forall i\in[n]} \sum_{i=1}^np_iy^{(i)}(\x^{(i)}+\dl^{(i)})^{\top}\widetilde{\w}_T~.
\]
Combining the above two inequalities, we obtain
$
\|\widetilde{\w}_T\|_2\geq \frac{\sum_{t=1}^T\alpha_t}{4}\gamma^2_{2,s}~,   
$
provided with the average regret bound $C_T\leq \frac{\gamma^2_{2,s}}{4}$.  Next, we study the directional error. 
Note that $m(\w)$ is 1-strongly concave with respect to the $\|\cdot\|_2$-norm. Based on the definition of $\w^*_{2,s}$, it is easy to draw the conclusion that $m(\w)$ is  maximized at $\gamma_{2,s}\w_{2,s}^*$. Note that $\overline{\w}_T=\frac{\widetilde{\w}_T}{\sum_{t=1}^T\alpha_t}$, and we have 
$ \frac{1}{2}\left\|\overline{\w}_T-\gamma_{2,s}\w^*_{2,s}\right\|_2^2\leq m(\w^*_{2,s})-m(\overline{\w}_T)\leq C_T,$
where the second inequality is based on Lemma \ref{lem:bound on m(w)}. On the other hand, following the same arguments as in \eqref{eqn:appb:11111}, we have 
\begin{align*}
\left\|\frac{\widetilde{\w}_T}{\|\widetilde{\w}_T\|}-\w_{2,s}^* \right\|_2 = \left\|\frac{\overline{\w}_T}{\|\overline{\w}_T\|}-\w_{2,s}^* \right\|_2 
\leq\frac{2\|\overline{\w}_T-\gamma_{2,s}\w_{2,s}^*\|_2}{\|\overline{\w}_T\|_2}.
\end{align*}
Therefore, 
\[
    \left\|\frac{\widetilde{\w}_T}{\|\widetilde{\w}_T\|}-\w_{2,s}^* \right\|_2\leq   \min\left\{\frac{\sum_{t=1}^T\alpha_t}{\|\widetilde{\w}_T\|_2},\frac{4}{\gamma_{2,s}^2}\right\}{2\sqrt{2C_T}}~.
\]


\subsection{Proof of Theorem \ref{thm:GDBAT}}
\label{appendix:proof of theorem GDBT}
We first focus on the algorithm equivalence. 
by setting $c_{t-1}=\frac{1}{\alpha_{t-1}}$, we have 
\begin{equation}
\label{eqn:proof:thm1:w_t}
    \w_{t}=\w_{t-1}-c_{t-1}\alpha_{t-1}\nabla h_{t-1}(\w_{t-1})=\sum_{i=1}^np_{t-1,i}y^{(i)}\x^{(i)}+\sum_{i=1}^n\frac{1}{n}y^{(i)}\dl_{t-1}^{(i)}.
\end{equation}

Next, for the $\dl^{(i)}$-player, we have 
\begin{equation}
    \begin{split}
    \label{eqn:delta:eqeq}
\dl^{(i)}_t =\argmax\limits_{\|\dl\|_s\leq\epsilon} \exp\left(-y^{(i)}(\x^{(i)}+\dl)^{\top}\widetilde{\w}_t\right).     
    \end{split}
\end{equation}
Note that the set of  $\argmin_{\|\dl\|_s\leq \epsilon}y^{(i)}\dl^{\top}\widetilde{\w}_t$ might not be unique.  Also note that, when $y^{(i)}=1$, $\dl_t^{(i)} \in\argmin_{\|\dl\|_s\leq \epsilon}\dl^{\top}\widetilde{\w}_t$, and when $y^{(i)}=-1$, $\dl_t^{(i)}\in\argmin_{\|\dl\|_s\leq \epsilon}-\dl^{\top}\widetilde{\w}_t$.  
Let $\widetilde{\dl}_t\in\argmin\limits_{\|\dl\|_s\leq \epsilon}\dl^{\top}\widetilde{\w}_t$. Based on Holder's inequality, we know that 
$-\widetilde{\dl}_t\in\argmin\limits_{\|\dl\|_s\leq \epsilon}-\dl^{\top}\widetilde{\w}_t$. Therefore, $y^{(i)}\dl^{(i)}_t$ serves as a universal solution for all $\dl^{(i)}$-players. We assume each  $\dl^{(i)}$-player adopts this solution, so  $y^{(i)}\dl^{(1)}=\dots=y^{(n)}\dl^{(n)}=\widetilde{\dl}_t,$ and thus 
$\sum_{i=1}^n\frac{1}{n}y^{(i)}\dl^{(i)}_{t-1}=\sum_{i=1}^np_{t-1,i}y^{(i)}\dl^{(i)}_{t-1},$
which, combined with \eqref{eqn:proof:thm1:w_t}, implies that 
$\w_t = \sum_{i=1}^np_{t-1,i}y^{(i)}(\x^{(i)}+\dl^{(i)}_{t-1}).$
To complete the proof, we focus on the $\p$-player and provide an expression for each term $p_{t,i}$. Based on the update rule and the relationship between FTRL with the negative entropy regularizer and Hedge \citep[for example]{orabona2019modern}, we have $\forall i\in[n]$,
\[
p_{t,i}= 
 \frac{
\exp\left(-y^{(i)}\x^{(i)\top}\widetilde{\w}_t\right)}{\sum_{k=1}^n\exp\left(-y^{(k)}\x^{(k)\top}\widetilde{\w}_t\right)}  
=  \frac{
\exp\left(-y^{(i)}(\x^{(i)}+\dl^{(i)}_{t})^{\top}\widetilde{\w}_t\right)}{\sum_{k=1}^n\exp\left(-y^{(k)}(\x^{(k)}+\dl^{(k)}_{t})^{\top}\widetilde{\w}_t\right)},
\]
where the equality is based on the fact that all $y^{(i)}\dl_t^{(i)}$ are equal with each other for $i\in[n]$. Substituting this above gives us
\begin{align*}
\w_t = \sum_{i=1}^n\Biggl(\frac{\exp\bigl(-y^{(i)}(\x^{(i)}+\dl^{(i)}_{t-1})^{\top}\widetilde{\w}_{t-1}\bigl)}{\sum_{k=1}^n\exp\bigl(-y^{(k)}(\x^{(k)}+\dl^{(k)}_{t-1})^{\top}\widetilde{\w}_{t-1}\bigl
)}y^{(i)}\left(\x^{(i)}+\dl^{(i)}_{t-1}\right)\Biggl).
\end{align*}
Finally, we relate the above expression to the iterates $\v_t$ realized by GD-AT.
Based on the definition of the exponential training loss $L$, we have 
$$\textstyle \frac{\nabla L(\widetilde{\w}_{t-1};\widetilde{\S}_{t-1})}{ L(\widetilde{\w}_{t-1};\widetilde{\S}_{t-1})}=-\sum_{i=1}^n\left(\frac{\exp\left(-y^{(i)}(\x^{(i)}+\dl^{(i)}_{t-1})^{\top}\widetilde{\w}_{t-1}\right)}{\sum_{k=1}^n\exp\left(-y^{(k)}(\x^{(k)}+\dl^{(k)}_{t-1})^{\top}\widetilde{\w}_{t-1}\right)}y^{(i)}\left(\x^{(i)}+\dl^{(i)}_{t-1}\right)\right)=-\w_t,$$
so $\widetilde{\w}_T=\widetilde{\w}_{T-1}+\alpha_T\w_T=\widetilde{\w}_{T-1}-\alpha_T\frac{\nabla L(\widetilde{\w}_{t-1};\widetilde{\S}_{T-1})}{ L(\widetilde{\w}_{t-1};\widetilde{\S}_{T-1})}.$
Recall that for each training example $(\x^{(i)}+\dl^{(i)}_t,y^{(i)})$ in $\widetilde{\mathcal{S}}_t$, the adversarial perturbation through GD-AT is defined as $\dl^{(i)}_t=\argmax\limits_{\|\dl\|_s\leq\epsilon} \exp\left(-y^{(i)}(\x^{(i)}+\dl)^{\top}\widetilde{\w}_t\right)$ which exactly matches the adversarial perturbation we obtain through the game framework in.  Thus, the equivalence part of the proof is finished by replacing $\widetilde{\w}_T$ with $\v_T$, and $\frac{\alpha_t}{L(\widetilde{\w}_{t-1};\widetilde{S}_{t-1})}$ with $\eta_{t-1}$. Next, we focus on the regret bounds, beginning with computing the regret bound for the $\w$-player. Recall that the $\w$-player uses online gradient descent. Let $\u=\argmin\limits_{\w\in\R^d}\sum_{t=1}^T\alpha_th_t(\w).$ 

Based on Theorem 6.10 of \citet{orabona2019modern}, we have  
\begin{align}
\sum_{t=1}^{T}& \alpha_th_t(\w_t) -  \sum_{t=1}^T \alpha_t h_t(\u) 
\leq  \frac{\|\u\|_2^2}{2c_{T+1}} + \sum_{t=1}^T\frac{c_{t}\alpha_t^2}{2}\left\|\nabla h_{t}(\w_t) \right\|^2_2\notag\\
=  {} & \frac{\|\u\|_2^2}{2c_{T+1}} + \frac{1}{2}\sum_{t=1}^T \alpha_t \left\|-\sum_{i=1}^np_{t,i}y^{(i)}\x^{(i)}-\sum_{i=1}^n\frac{1}{n}y^{(i)}\dl^{(i)}_{t}+\w_{t}\right\|_2^2\notag\\
= {} & \frac{\|\u\|_2^2}{2c_{T+1}} + \frac{1}{2}\sum_{t=1}^T \alpha_t \left\|-\sum_{i=1}^np_{t,i}y^{(i)}\x^{(i)}-\sum_{i=1}^n\frac{1}{n}y^{(i)}\dl^{(i)}_{t}+\sum_{i=1}^np_{t-1,i}y^{(i)}\x^{(i)}+\sum_{i=1}^n\frac{1}{n}y^{(i)}\dl^{(i)}_{t-1}\right\|_2^2\notag\\
\leq  {} & \frac{\|\u\|_2^2}{2c_{T+1}} +  \sum_{t=1}^T\alpha_t\|\p_t-\p_{t-1}\|_1^2 + \frac{1}{n}\sum_{i=1}^n\sum_{t=1}^T\alpha_t\|y^{(i)}\dl^{(i)}_{t}-y^{(i)}\dl^{(i)}_{t-1}\|_2^2~. \label{eqn:ATGD1111112}
\end{align}
To bound the first term, we focus on $\u$. 

We have  
\[
    \begin{split}
\u = {} & \frac{1}{\sum_{t=1}^T\alpha_t}\left(\sum_{t=1}^T\alpha_tA^{\top}\p_t+\frac{1}{n}\sum_{i=1}^n\sum_{t=1}^T\alpha_ty^{(i)}\dl^{(i)}_t\right)
= A^{\top}\overline{\p}_T +\frac{1}{n}\sum_{i=1}^ny^{(i)} \overline{\dl}^{(i)}_T,    
    \end{split}
\]
where $\overline{\p}_T$ and $\overline{\dl}^{(i)}_T$ are the weighted average of $p_t$ and $\dl_t^{(i)}$. 
Note that $\p_t\in\Delta^n$, and $\|\dl^{(i)}_t\|_s\leq \epsilon$, which are two convex sets, so it is easy to verify that $\|A^{\top}\overline{\p}_t\|_2\leq 1$. For the second term, if $s\in(1,2]$, then 
$\|\overline{\dl}_T^{(i)}\|_2\leq \|\overline{\dl}_T^{(i)}\|_s\leq \epsilon$. Thus, we have $\|\u\|_2\leq 1+\epsilon$. 
{ Later, we will show that the second term of \eqref{eqn:ATGD1111112} can be cancelled by the $\p$-player's regret. It remains to bound the third term. We first introduce the following lemma.
\begin{lemma}[{\citet[Lemma 18,][]{wang2021no}}] 
\label{lem:key:path:delta}
Let $\mathcal{K}$ be a $\lambda$-strongly convex set with respect to some norm $\|\cdot\|$. Let $\m,\mathbf{n}\in \R^d$ be two vectors, and 
$$
\textstyle
\r=\argmax\limits_{\r'\in\K} \r'^{\top}\m~,\qquad  
\s=\argmax\limits_{\s'\in\K} \s'^{\top}\mathbf{n}~. 
$$
Then $\|\r-\s\| \leq \frac{2\|\m-\mathbf{n}\|}{\lambda(\|\m\|+\|\mathbf{n}\|)}.$ 
\end{lemma}
When $s\in(1,2]$, the set $\{\dl|\|\dl\|_s\leq \epsilon\}$ is $\frac{s-1}{\epsilon}$-strongly convex with respect to the $\ell_s$-norm \citep{garber2015faster}. 

Therefore, we have $\forall i\in[n], t>1$,
\begin{align}
 \|y^{(i)}\dl_t^{(i)}  - y^{(i)}\dl_{t-1}^{(i)}\|_2^2 &=    \|\dl_t^{(i)}  - \dl_{t-1}^{(i)}\|_2^2 
 \leq \|\dl_t^{(i)}  - \dl_{t-1}^{(i)}\|_s^2\notag \\
 \overset{\eqref{eqn:delta:eqeq}}{=} {} & \left\|\argmin \limits_{\|\dl\|_s\leq \epsilon} y^{(i)}\dl^{\top}\widetilde{\w}_t  - \argmin \limits_{\|\dl\|_s\leq \epsilon} y^{(i)}\dl^{\top}\widetilde{\w}_{t-1}\right\|_s^2\notag\\
 = {} & \left\|\argmax\limits_{\|\dl\|_s\leq \epsilon} -y^{(i)}\dl^{\top}\widetilde{\w}_t  - \argmax\limits_{\|\dl\|_s\leq \epsilon} -y^{(i)}\dl^{\top}\widetilde{\w}_{t-1}\right\|_s^2\notag\\
 {\leq} {} & \left[\frac{\epsilon\|\widetilde{\w}_t-\widetilde{\w}_{t-1}\|_s}{(s-1)( \|\widetilde{\w}_t\|_s+ \|\widetilde{\w}_{t-1}\|_s)}\right]^2
 = \left[\frac{\epsilon\alpha_t\|\w_t\|_s}{(s-1) \|\widetilde{\w}_t\|_s}\right]^2~.\label{eqn:bound:path:delta}
\end{align}
To proceed, we show an upper bound of $\|\w_t\|_s$, and a lower bound for $\|\widetilde{\w}_t\|_s.$ We have 
$$
\|\w_t\|_s=\left\| \sum_{i=1}^np_{t-1,i}y^{(i)}(\x^{(i)}+\dl^{(i)}_{t-1}) \right\|_s\leq d^{\frac{1}{s}-\frac{1}{2}}+\epsilon~,
$$
\begin{equation}
    \begin{split}
    \label{eqn:advtalowboundforwwww}
        \|\widetilde{\w}_t\|_s\geq  \|\widetilde{\w}_t\|_2=  \|\w^*\|_2\|\widetilde{\w}_T\|_2 \geq \w^{*\top}\widetilde{\w}_{t}  = \sum_{i=1}^t\alpha_i\w^{*\top}\w_i
 \geq \frac{t\gamma_2}{2}~.
    \end{split}
\end{equation}
To summarize, we get  $\|y^{(i)}\dl_t^{(i)}  - y^{(i)}\dl_{t-1}^{(i)}\|_2^2 \leq \frac{(d^{\frac{1}{s}-\frac{1}{2}}+\epsilon)^2\epsilon^2}{t^2(s-1)^2\gamma_2^2}$ for $t> 1$, and $\|y^{(i)}\dl_1^{(i)}  - y^{(i)}\dl_{0}^{(i)}\|_2^2 \leq  \epsilon^2$ for $t=1$.

Thus 
\begin{equation*}
\frac{1}{n}\sum_{i=1}^n\sum_{t=1}^T   \|y^{(i)}\dl_t^{(i)}  - y^{(i)}\dl_{t-1}^{(i)}\|_2^2 \leq \frac{\pi \epsilon^2(d^{\frac{1}{s}-\frac{1}{2}}+\epsilon)^2}{6(s-1)^2\gamma_2^2}+\epsilon^2.
\end{equation*}
Combined with \eqref{eqn:ATGD1111112}, we have 
\begin{equation*}
    \begin{split}
  \sum_{t=1}^{T} {} &\alpha_th_t(\w_t) -  \min\limits_{\w\in\R^d}\sum_{t=1}^T \alpha_t h_t(\w) \leq (1+\epsilon)^2 +  \frac{\pi \epsilon^2(d^{\frac{1}{s}-\frac{1}{2}}+\epsilon)^2}{6(s-1)^2\gamma_2^2}+\epsilon^2 + \frac{1}{2}\sum_{t=1}^T\|\p_t-\p_{t-1}\|_1^2~.      
    \end{split}
\end{equation*}
}

Next, for the $\p$-player, since it uses FTRL$^+$  \citep{wang2021no}, we have
$\sum_{t=1}^T\alpha_t\ell_{t}(\p_t)-\sum_{t=1}^T\ell_{t}(\p^*)\leq \log n - \sum_{t=1}^T\frac{1}{2}\|\p_t-\p_{t-1}\|_1^2,$ and for the $\dl^{(i)}$-player {since it uses the FTL$^+$ algorithm, we know its regret bounded by 0. The proof is finished by applying Theorem \ref{thm:main:normalized margin}, combining  the regret bound for all players, and \eqref{eqn:advtalowboundforwwww}.}

%


\subsection{Proof of Theorem \ref{thm:acccccc}}
\label{sec:proof:faster rates}
We first focus on the algorithm equivalence. For the $\w$-player, similar to the proof in Section \ref{appendix:Proof of section 3},


we can obtain the following closed form:
\begin{equation}
    \begin{split}
    \label{eqn:wwww:accc}
    \w_t =  \frac{1}{\sum_{j=1}^t\alpha_j}\sum_{j=1}^t \alpha_j \left(\sum_{i=1}^np_{j,i}y^{(i)}\x^{(i)}+\sum_{i=1}^n \frac{1}{n}y^{(i)}\dl_{j}^{(i)}\right). 
    \end{split}
\end{equation}
For the $\dl^{(i)}$-player, we have $\dl^{(i)}_t =\argmax\limits_{\|\dl\|_s\leq \epsilon}\exp\left(-y^{(i)}\left(\sum_{j=1}^{t-1}\alpha_j\w_j+\alpha_t\w_{t-1}\right)^{\top}\dl \right)$. Finally, we focus on the $\p$-player. Following similar arguments as in \eqref{eqn:delta:eqeq}, we know $\forall i\in[n]$, $y^{(i)}\dl^{(i)}_t$ are equivalent to each other. Based on the relationship between Hedge and FTRL, we have
\begin{equation*}
    \begin{split}
p_{t,i}= \frac{
\exp\left(-y^{(i)}(\x^{(i)}+\dl^{(i)}_{t})^{\top}\left(\sum_{j=1}^{t-1}\alpha_j\w_j+\alpha_t\w_{t-1}\right)\right)}{\sum_{k=1}^n\exp\left(-y^{(k)}(\x^{(k)}+\dl^{(k)}_{t})^{\top}\left(\sum_{j=1}^{t-1}\alpha_j\w_j+\alpha_t\w_{t-1}\right)\right)}~.
    \end{split}
\end{equation*}
Let $\widetilde{\mathcal{S}}_t$ contains all $\widetilde{\x}^{(i)}_t=\x^{(i)}+\dl^{(i)}_t$.  Then
\( 
\frac{\nabla L(\widehat{\w}_t;\widetilde{\mathcal{S}}_t)}{L(\widehat{\w}_t;\widetilde{\mathcal{S}}_t)}
=
\Bigl(\sum_{i=1}^np_{t,i}y^{(i)}\x^{(i)}+\sum_{i=1}^n\frac{1}{n}y^{(i)}\dl_t^{(i)}\Bigl).
\)
Note that $\alpha_t=\frac{t}{2}$. Let $\z_t=\w_t\sum_{j=1}^t\alpha_j$, and we have 
\begin{equation*}
    \begin{split}
\z_t\overset{\eqref{eqn:wwww:accc}}{=} {} & \z_{t-1}+\alpha_t\left(\sum_{i=1}^np_{t,i}y^{(i)}\x^{(i)}+\sum_{i=1}^n\frac{1}{n}y^{(i)}\dl_t^{(i)}\right) 
=  \z_{t-1} - \alpha_t \frac{\nabla L(\widehat{\w}_{t};\widetilde{\mathcal{S}}_t)}{L(\widehat{\w}_{t};\widetilde{\mathcal{S}}_t)}.
    \end{split}
\end{equation*}
Moreover,  
$ \widehat{\w}_t = \widetilde{\w}_{t-1} + \alpha_t \w_{t-1} = \widetilde{\w}_{t-1}+ \frac{2}{t-1}\z_{t-1}, $
and $\widetilde{\w}_t= \widetilde{\w}_{t-1} + \alpha_t\w_t = \widetilde{\w}_{t-1} + \frac{2}{t+1} \z_t.$      

To summarize, we get:
\begin{equation*}
\text{For each round}\ t:  
\begin{cases}
    & \widehat{\w}_t = \widetilde{\w}_{t-1}+ \frac{2}{t-1}\z_{t-1}\\
    & \dl_t^{(i)} = \argmax_{\|\dl\|\leq \epsilon} \exp\left(-y^{(i)}\left(\x^{(i)}+\dl\right)^{\top}\widehat{\w}_t\right)\,\,\forall i\in[n]  \\
    & \z_t = \z_{t-1} - \frac{t}{2L(\widehat{\w}_t;\S_t)}\nabla L(\widehat{\w}_t;\S_t)\\
    & \widetilde{\w}_t = \widetilde{\w}_{t-1}+\frac{2}{t+1}\z_t~.
    \end{cases}        
\end{equation*}
The proof is finished by replacing $\widehat{\w}_t$ with $\widehat{\v}_t$, $\widetilde{\w}_t$ with $\v_t$, and 
setting $\beta_{t,1}=1,\beta_{t,2}=\frac{2}{t-1}$, $\eta_{t-1}=\frac{t}{2L(\widehat{\w}_t;\S_t)}$, $\beta_{t,3}=1$, and $\beta_{t,4}=\frac{2}{t+1}$.

Next, we focus on the regret bounds. For the $\w$-player, note that $h_i(\w)$ is 1-strongly convex. Therefore, by conducting FTL$^+$ \citep{orabona2019modern}, we have 
$\text{Reg}_T^{\w} = -\sum_{t=1}^T\frac{t(t-1)}{4}\|\w_t-\w_{t-1}\|^2_2.  $
For the $\p$-player, since it uses Optimistic FTRL with a $1$-strongly convex regularizer with respect to the $\ell_1$-norm, we have 
$\text{Reg}_T^{\p}\leq \log n + \frac{t^2\|\w_t-\w_{t-1}\|^2_2}{8}.$
{
For the $\dl^{(i)}$-player, it uses the optimistic FTL algorithm. Let $\widehat{\dl}^{(i)}_t=\argmin\limits_{\|\dl\|_s\leq \epsilon}\sum_{j=1}^{t-1}\alpha_js_j^{(i)}(\dl)$ 
be the solution for FTL. We have
\begin{align*}
 \text{Reg}^{\dl^{(i)}}_T 
 &\leq \sum_{t=1}^T y^{(i)}\alpha_t \left[ \left(\w_t^{\top}\dl_t^{(i)}-\w_t^{\top}\widehat{\dl}_{t+1}^{(i)}\right) - \left( \w_{t-1}^{\top}\dl_t^{(i)}-\w_{t-1}^{\top}\widehat{\dl}_{t+1}^{(i)}\right) \right] \\
 &= \sum_{t=1}^Ty^{(i)}\alpha_t \left( \w_t -\w_{t-1} \right)^{\top}\left(\dl_t^{(i)}-\widehat{\dl}_{t+1}^{(i)}\right)
 \leq \sum_{t=1}^T \frac{t^2\|\w_t-\w_{t-1}\|^2_2}{16} + \sum_{t=1}^T \bigl\| \dl_t^{(i)}-\widehat{\dl}_{t+1}^{(i)}\bigl\|^2_2~,
\end{align*}
where the first inequality is based on the regret of optimistic FTL (see, e.g.,   \citet[Lemma 9,][]{wang2021no}), and the last inequality is based  on Young's inequality. To proceed, we need to bound the second term at the RHS. Let $\widehat{\w}_{t} = \sum_{j=1}^{t-1}\alpha_j\w_j + \alpha_t\w_{t-1}$.
For $s\in(1,2]$, following similar procedures in \eqref{eqn:bound:path:delta}, applying Lemma \ref{lem:key:path:delta}, we have 
\begin{equation*}
 \|\dl_t^{(i)}  - \widehat{\dl}_{t+1}^{(i)}\|_2^2 
 \leq
 \left[\frac{\epsilon\|\widehat{\w}_t-\widetilde{\w}_{t}\|_s}{(s-1)( \|\widehat{\w}_t\|_s+ \|\widetilde{\w}_{t}\|_s)}\right]^2
 =  
 \left[\frac{\epsilon\alpha_t\|\w_{t-1}-\w_{t}\|_s}{(s-1) \|\widetilde{\w}_t\|_s}\right]^2.
\end{equation*}
We provide a tighter bound on $\|\widetilde{\w}_t\|_s$. Note that 
\begin{align*}
\|\widetilde{\w}_t\|_s\geq \|\widetilde{\w}_t\|_2 
&= \|\w^*\|_2\|\widetilde{\w}_t\|_2\geq \w^{*\top}\widetilde{\w}_t= \Bigl(\sum_{k=1}^{T}\alpha_k \w^{*\top}{\w}_k\Bigl) \\
&= \sum_{k=1}^{t}\alpha_k\left(\frac{1}{\sum_{j=1}^k\alpha_j} \sum_{j=1}^k\alpha_j\Bigl(\sum_{i=1}^np_{j,i}y^{(i)}\x^{(i)\top}\w^*+y^{(i)}\dl_j^{(j)\top}\w^*\Bigl)\right)\\
&\geq {(\gamma_2-\epsilon\|\w^*\|_r)\sum_{k=1}^t\alpha_k}
\geq 
2\gamma_2\sum_{k=1}^T\alpha_k/3
\geq 
t^2\gamma_2/3.    
\end{align*}
On the other hand, from \eqref{eqn:wwww:accc}, it is easy to see that
$\textstyle \|\w_t-\w_{t-1}\|_s\leq 2(d^{\frac{1}{s}-\frac{1}{2}}+\epsilon).$ Thus, 
\begin{equation*}
 \|\dl_t^{(i)}  - \widehat{\dl}_{t+1}^{(i)}\|_2^2 
 \leq \left[ \frac{2\epsilon (d^{\frac{1}{s}-\frac{1}{2}}+\epsilon)}{(s-1)t\gamma_2}\right]^2 = \frac{4\epsilon^2(d^{\frac{1}{s}-\frac{1}{2}}+\epsilon)^2}{(s-1)^2t^2\gamma_2^2}. 
\end{equation*}
Summing over $t$, we get $\sum_{t=1}^T \bigl\| \dl_t^{(i)}-\widehat{\dl}_{t+1}^{(i)}\bigl\|^2_2 \leq \frac{\pi\epsilon^2(d^{\frac{1}{s}-\frac{1}{2}}+\epsilon)^2}{(s-1)^2\gamma_2^2}.$
The proof is finished by combining the regret bounds, and noticing that when $t\geq 4$, $\frac{t(t-1)}{4}\geq \frac{t^2}{8} + \frac{t^2}{16}.$ If $s>2$, we can bound the regret of the last term by $\sum_{t=1}^T \bigl\| \dl_t^{(i)}-\widehat{\dl}_{t+1}^{(i)}\bigl\|^2_2\leq 4Td^{\frac{1}{2}-\frac{1}{s}}\epsilon^2~,$
so the regret of the $\dl^{(i)}$-player satisfies
$\text{Reg}_T^{\dl^{(u)}} \leq  \sum_{t=1}^T \frac{t^2\|\w_t-\w_{t-1}\|^2_2}{16} + 4T\epsilon^2d^{\frac{1}{2}-\frac{1}{s}}$~. 

We can finish the proof by combining the above regret with the regret bound of other players, and also the lower bound of $\widetilde{\w}_t$, and then apply Theorem \ref{thm:main:normalized margin}.\\
}

\section{Future directions}
Despite the effectiveness of the game framework in handling generic methods and adversarial training, it presently holds some limitations. First, the algorithmic equivalences are currently operational only for exponential loss; the extension to handle more general losses a vital area for future research. More generally, identifying algorithmic equivalence is nuanced and non-trivial, and it is as yet unresolved whether this framework can elucidate other methods, such as the last-iterate of MD, or MD with non-strongly convex norms.
It also remains to be seen whether more advanced adaptive online learning algorithms can be captured by our framework, such as parameter-free online learning  \citep{orabona2016coin,cutkosky2018black}.\\

\noindent \textbf{Acknowledgements}\ GW was supported by a ARC-ACO fellowship provided by Georgia Tech. VM was supported by the NSF (through CAREER award CCF-2239151 and award IIS-2212182), an Adobe Data Science Research Award, an Amazon Research Award and a Google Research Colabs Award. JA was supported by the AI4OPT Institute, as part of NSF Award 2112533, and NSF through Award IIS-1910077.




\bibliography{arXiv/ref}
\bibliographystyle{plainnat}

\end{document}